\documentclass{article}

\usepackage[preprint]{neurips_2026}

\usepackage[utf8]{inputenc} 
\usepackage[T1]{fontenc}    

\usepackage{hyperref}       
\usepackage{url}            
\usepackage{booktabs}       
\usepackage{amsfonts}       
\usepackage{nicefrac}       
\usepackage{microtype}      
\usepackage{xcolor}         

\usepackage{paper}          

\usepackage{etoolbox}       
\makeatletter
\patchcmd{\thm@space@setup}%
  {\thm@preskip=\topsep}   {\thm@preskip=4pt}{}{}
\patchcmd{\thm@space@setup}%
  {\thm@postskip=\thm@preskip}{\thm@postskip=4pt}{}{}
\makeatother

\usepackage{cleveref}       
\crefname{assumption}{assumption}{assumptions}   
\Crefname{assumption}{Assumption}{Assumptions}   
\usepackage{subcaption}
\usepackage{wrapfig}
\usepackage{longtable}
\usepackage{thmtools}
\usepackage{thm-restate}
\usepackage{placeins}

\title{Retrieval-Guided Fine-Tuning as Noisy Estimation: Risk bounds and Architectural Analysis}
\workshoptitle{New in ML Workshop @ NeurIPS 2026}
\author{%
  Bhargav Lad \\
  Department of Mathematics \\
  Indian Institute of Technology Delhi \\
  New Delhi, India \\
  \texttt{mt6240942@iitd.ac.in}
  \And
  Yifan Hao \\
  Department of Computer Science \\
  University of Illinois Urbana-Champaign \\
  Urbana, IL 61801 \\
  \texttt{yifanh12@illinois.edu}
}
\hypersetup{
  pdfauthor={Bhargav Lad, Yifan Hao},
  pdftitle={Retrieval-Guided Fine-Tuning as Noisy Estimation: Risk bounds and Architectural Analysis}
}

\begin{document}

\maketitle

\begin{abstract}

Retrieval-Guided Fine-Tuning (RAG-FT) incorporates retrieved data directly
into the training objective, but the statistical consequences of noisy
retrieval during training remain theoretically undercharacterized. We study
this question by modeling RAG-FT as an estimation problem in a multi-task
linear regression framework, using an OLS proxy for single-layer
linear self-attention to obtain finite-sample risk bounds. Under homoscedastic
retrieval noise, we show that retrieval failure decays exponentially with
task separation relative to noise, and derive explicit finite-sample
conditions under which RAG-FT achieves lower risk than both target-only and
full-corpus training. We then introduce a Distance-Proportional Noise (DPN)
model, in which retrieval quality degrades with rank, and compare two
estimators under the same retrieval process: the OLS proxy and the
literal, uniform-weight forward pass of linear self-attention. We prove that
the attention estimator's bias diverges as $\Theta(n^{2q})$ even under
exact retrieval, while OLS risk remains $\Theta(d/n)$ for every noise
exponent $q>0$. These results locate the instability not in noisy retrieval
itself, but in the fixed, unweighted aggregation of the literal LSA forward
pass, which reweighting by reliability empirically removes. We validate
the predicted rate separation through direct simulation of the DPN model.

\end{abstract}

\section{Introduction}

Retrieval-Augmented Generation (RAG) grounds language model predictions in
external knowledge by conditioning on retrieved documents at inference
time~\citep{lewis2020retrieval}. Because model parameters remain fixed at
deployment, generalization depends entirely on what fits in the context
window at test time, and performance degrades when retrieved passages are
topically related but parametrically misaligned with the
query~\citep{levy2024same, liu2024lost}. This has motivated
\textbf{Retrieval-Guided Fine-Tuning (RAG-FT)}, which incorporates retrieved
documents directly into the fine-tuning objective rather than the
inference-time context, forcing parametric adaptation to retrieved
data~\citep{zhang2024raft, patil2024gorilla, asai2024self}. Despite its
empirical adoption, the statistical mechanics of RAG-FT remain
uncharacterized: whether moving retrieval from inference time to training
time changes the generalization limits of noisy retrieval is an open
question. The closest theoretical treatment~\citep{guo2025retrieval}
analyzes inference-time RAG under frozen weights; we ask what changes once
the retrieved data instead drives a parameter update.

We study this question in a multi-task linear regression framework where a
retriever selects which task's data to pool with the target dataset for
fine-tuning. Motivated by~\citet{zhang2024trained}, who show that linear
self-attention trained on an exchangeable context converges to the ordinary
least squares solution, we assume the fine-tuned estimator coincides with
this OLS solution on the pooled data (an OLS proxy). This decouples
the statistics of retrieval from optimization dynamics and lets us derive
exact, finite-sample risk bounds.

\paragraph{Contributions.} We formalize RAG-FT as a stochastic estimation
problem in a multi-task linear regression framework, using an OLS
proxy for single-layer linear self-attention.
\begin{enumerate}[label=(\roman*)]
\item We give explicit, finite-sample conditions on task separation and
retrieval accuracy under which RAG-FT achieves strictly lower expected risk
than both target-only and full-corpus training, compared against each
baseline's true risk rather than a looser bound (Theorem~\ref{thm:dominance}).
\item Under a distance-proportional noise model, we show the OLS proxy's
risk remains $\Theta(d/n)$ for every noise exponent $q>0$, while the literal,
uniform-weight linear self-attention forward pass has bias diverging as
$\Theta(n^{2q})$ regardless of retrieval outcome or ambient dimension
(Theorems~\ref{thm:ols-dpn},~\ref{thm:lsa-dpn}).
\item We validate both rate separations, the dimension-independence of the
divergence, and its recovery under a reweighted forward pass, via direct
simulation of the generative model (Section~\ref{sec:empirics}).
\end{enumerate}

\section{Related work}

A line of work on in-context learning shows that transformers trained on
linear-regression instances implicitly execute ordinary least squares
in-context~\citep{garg2022can, akyurek2022learning, von2023transformers},
that trained linear self-attention provably converges to this OLS solution
but fails under covariate shift~\citep{zhang2024trained}, and that this
convergence emulates preconditioned gradient descent~\citep{ahn2023transformers,
mahankali2024one}; we adopt this OLS proxy but depart from the i.i.d.\
setting by modeling the heteroscedastic noise and distributional shift that
stochastic retrieval introduces. Separately, retrieval-augmented generation
has moved from an inference-time mechanism~\citep{lewis2020retrieval}, known
to be susceptible to noisy or irrelevant context~\citep{levy2024same,
liu2024lost, yoran2023making}, toward methods that push retrieval directly
into training: RAFT~\citep{zhang2024raft}, Retriever-Aware
Training~\citep{patil2024gorilla}, Self-RAG~\citep{asai2024self}, context
compression~\citep{xu2023recomp}, and end-to-end retrieval-augmented
training~\citep{izacard2023atlas}, none of which give a formal account of
how retrieval reshapes the bias-variance tradeoff during parameter updates.
The closest theoretical precursor is~\citep{guo2025retrieval}, which derives
finite-sample generalization bounds for \emph{inference-time} RAG under
distance-proportional noise and establishes a retrieval ceiling
$n^*=\Theta(d^{1/(2q)})$ for a frozen-weight estimator; we study the same
distance-proportional noise model once retrieval instead drives a parameter
update.
Our framework further builds on excess-risk bounds for multi-task and
heterogeneous representation learning~\citep{tripuraneni2021provable,
maurer2016benefit, JMLR:v13:solnon12a}, on debiasing techniques for noisy or
label-corrupted ICL~\citep{liang2025dual}, and on Bayesian views of
in-context algorithm selection~\citep{bai2023transformers, ahuja2023closer},
which together motivate treating retrieval as a stochastic selection
operator with a quantifiable failure rate.

\section{Problem setup}
\label{sec:problem_setup}

\paragraph{Multi-task data model.} We consider a multi-task linear regression
framework with $K \geq 2$ tasks, indexed by $k \in [K] := \{1, \dots, K\}$,
each associated with a parameter vector $\theta_k \in \mathbb{R}^d$
satisfying $\|\theta_k\| \leq B$. Let $k^\star \in [K]$ denote the (unknown)
target task. Data for task $k$ are generated i.i.d.\ from the Gaussian
linear model
\[
y = x^\top \theta_k + \varepsilon, \quad x \sim \mathcal{N}(0, I_d), \quad
\varepsilon \sim \mathcal{N}(0, \sigma^2), \quad x \perp \varepsilon.
\]
Task $k$ contributes a dataset $D_k = \{(x_i^{(k)}, y_i^{(k)})\}_{i=1}^n$,
and the query (target) dataset $D' = \{(x_j', y_j')\}_{j=1}^m$ is generated
from $\theta_{k^\star}$; all datasets $\{D_k\}_{k=1}^K$ and $D'$ are
mutually independent. We quantify task heterogeneity via the minimum and
maximum distance from the target task, $\Delta_{\min} :=
\min_{k \neq k^\star} \|\theta_k - \theta_{k^\star}\|_2$ and
$\Delta_{\max} := \max_{k \neq k^\star} \|\theta_k - \theta_{k^\star}\|_2$,
which respectively set the difficulty of identifying the target task under
noise and bound the worst-case bias under retrieval failure.

\paragraph{Retrieval operator.} A retriever selects an index
$\hat{k} \in [K]$ prior to fine-tuning, using only the observed data
$\{D_k\}_{k=1}^K$ and $D'$; the true parameters $\{\theta_k\}$ and $k^\star$
are never observed by the retriever or the estimator. Here we treat the
retriever as a black box, characterized entirely by its failure probability
$\delta := P(\hat{k} \neq k^\star)$; equivalently, let $R \in \{0,1\}$ be
Bernoulli with $P(R=1) = 1-\delta$, where $R=1$ denotes correct retrieval,
so that the retrieved dataset is $D_R = D_{k^\star}$ if $R=1$ and
$D_R = D_k$ for some $k \neq k^\star$ if $R=0$. The results of this section
hold for any retriever satisfying this definition; Section~\ref{subsec:failure_probability} later fixes
a concrete decision rule, based on comparing per-task OLS estimates, and
derives an explicit finite-sample bound on $\delta$ for it.

\paragraph{RAG-FT estimator and OLS proxy.} Given $D'$ and $D_R$, the
pooled dataset is $D_{\text{pool}} := D' \cup D_R$, and the
retrieval-guided fine-tuning estimator is the empirical risk minimizer
$\hat\theta_{\mathrm{FT}} := \arg\min_\theta \sum_{(x,y) \in D_{\text{pool}}}
(y - x^\top \theta)^2$. Because $\hat\theta_{\mathrm{FT}}$ is obtained by optimizing
a non-convex objective under architectures such as linear self-attention,
and because retrieval induces statistical dependence between $D'$ and
$D_R$ that violates standard i.i.d.\ assumptions, we decouple the
statistics of retrieval from the optimization dynamics via the following
structural assumption.

\begin{assumption}[OLS Proxy]
\label{assumption:oracle_ols_proxy}
The estimator $\hat\theta_{\mathrm{FT}}$ coincides with the exact ordinary least
squares solution on the pooled dataset:
\[
\hat\theta_{\mathrm{FT}} = (X_{\text{pool}}^\top X_{\text{pool}})^{-1}
X_{\text{pool}}^\top Y_{\text{pool}},
\]
where $(X_{\text{pool}}, Y_{\text{pool}})$ are the design matrix and
response vector formed from $D_{\text{pool}}$.
\end{assumption}

\paragraph{Risk functional.} For a fresh sample $(x,y)$ drawn from the
target model $y = x^\top \theta_{k^\star} + \varepsilon$, the prediction
risk of an estimator $\hat\theta$ is
$\mathcal{R}(\hat\theta) := \mathbb{E}_{(x,y)}[(y - x^\top \hat\theta)^2]$. Our
objective is to characterize $\mathbb{E}[\mathcal{R}(\hat\theta_{\mathrm{FT}})]$, jointly over
the data-generating process and the retrieval randomness. Since
$\hat\theta_{\mathrm{FT}}$ depends on $R$, the law of total expectation gives
\[
\mathbb{E}[\mathcal{R}(\hat\theta_{\mathrm{FT}})] = (1-\delta)\,
\mathbb{E}[\mathcal{R}(\hat\theta_{\mathrm{FT}}) \mid R=1] + \delta\,
\mathbb{E}[\mathcal{R}(\hat\theta_{\mathrm{FT}}) \mid R=0],
\]
separating the variance reduction from pooling homogeneous target-task
data ($R=1$) from the bias incurred by pooling heterogeneous data under
retrieval failure ($R=0$).
\section{Theoretical analysis: under homoscedastic retrieval noise}
\label{sec:uniform_noise}

To rigorously evaluate the generalization of Retrieval-Guided Fine-Tuning (RAG-FT), we first isolate the statistical penalty of task mismatch from distance-dependent noise by assuming a uniform homoscedastic retrieval regime. This controlled setting quantifies how selective retrieval navigates the fundamental bias-variance tradeoff in multi-task environments.

We first bound the retrieval failure probability (Section~\ref{subsec:failure_probability}), formalize the expected prediction risk (Section~\ref{subsec:risk_bounds}), and define the exact conditions where RAG-FT achieves a statistical advantage over standard baselines (Section~\ref{subsec:comp_with_baselines}).

\subsection{Finite-sample risk decomposition (failure probability)}
\label{subsec:failure_probability}

\begin{assumption}[\textbf{Uniform Homoscedastic Noise}]
\label{assumption:uniform_noise}
For all candidate tasks $k \in [K]$, we assume the task datasets $D_k$ are corrupted by independent, identically distributed Gaussian noise $\epsilon \sim \mathcal{N}(0, \sigma^2)$.
\end{assumption}

By fixing the variance scale across all datasets, this assumption strictly isolates the generalization penalty to the geometric bias induced by task mismatch, decoupling it from heteroscedastic effects (addressed later in Section~\ref{sec:non_uniform_noise}).

\begin{restatable}[\textbf{Distribution of OLS Estimation Noise}]{lemma}{lemmaone}
\label{lemma:finite_sample_dist_of_ols_estimation_noise}
Under Gaussian design $x_i \sim \mathcal{N}(0, I_d)$ and homogeneous noise $\epsilon_i \sim \mathcal{N}(0, \sigma^2)$, the OLS estimator for any task $k$ (including the target query $q$) satisfies $\hat{\theta}_k - \theta_k \mid X_k \sim \mathcal{N}\left(0, \sigma^2 (X_k^\top X_k)^{-1}\right)$.
\end{restatable}

\begin{restatable}[\textbf{Signal-Fluctuation Decomposition}]
{lemma}{lemmatwo}
\label{lemma:signal_fluc_decomp}
For any incorrect candidate task $k \neq k^\ast$, let $T_k$ denote the empirical distance gap evaluated by the nearest-estimator retrieval rule. This decision statistic decomposes exactly into a deterministic geometric signal and two stochastic error terms:
\[
T_k :=
\underbrace{
\left\|\hat{\theta}_k - \hat{\theta}_q\right\|_2^2
-
\left\|\hat{\theta}_{k^*} - \hat{\theta}_q\right\|_2^2
}_{\text{Empirical distance gap}}
=
\underbrace{
\left\|\theta_k - \theta_{k^*}\right\|_2^2
}_{\text{Geometric signal}}
+
\underbrace{L_k}_{\text{Linear fluctuation}}
+
\underbrace{Q_k}_{\text{Quadratic fluctuation}}
\]
where $L_k$ and $Q_k$ are the estimation noise fluctuations bounded in auxiliary Lemmas~\ref{lemma:linear_fluc_term} and ~\ref{lemma:quad_fluc_conc} (Appendix~\ref{sec:B}).
\end{restatable}

Bounding the sub-Gaussian tail of $L_k$ (Lemma~\ref{lemma:linear_fluc_term}) and the sub-exponential tail of $Q_k$ (Lemma~\ref{lemma:quad_fluc_conc}) around the geometric signal $\|\theta_k - \theta_{k^\ast}\|_2^2$ yields the finite-sample failure probability.

\begin{restatable}[\textbf{Finite-Sample Bound on Retrieval Failure}]{theorem}{RetrievalFailureProb}
\label{thm:delta_ineq}
Let the minimum task separation satisfy $\Delta_{\min}^2 \ge C_1 \sigma^2 d/n$ for a constant $C_1>0$. The retrieval failure probability $\delta = \mathbb{P}(\hat k \neq k^*)$ satisfies
$$
\delta \leq 3(K-1)\exp\!\left( -C\min\!\left\{ \frac{\Delta_{\min}^4}{\bar{\nu}^{\,2}}, \frac{\Delta_{\min}^2}{\bar{\alpha}} \right\} \right) + 3K\exp(-c'd),
$$
where $C,c'>0$ are universal constants, $\bar{\nu}^{\,2}=O\!\left(\sigma^2\Delta_{\max}^2/n + d\sigma^4/n^2\right)$, and $\bar{\alpha}=O\!\left(\sigma\Delta_{\max}/\sqrt n + \sigma^2/n\right)$.
\end{restatable}

\noindent
Assuming target sample size $m = \Omega(n)$, query-side estimation noise is absorbed into $\bar\nu^2$ and $\bar\alpha$. Theorem~\ref{thm:delta_ineq} demonstrates that failure probability decays exponentially once the geometric task separation $\Delta_{\min}^2$ overcomes the finite-sample estimation noise scale, leaving only an irreducible $O(K\exp(-c'd))$ probability of ill-conditioned designs. Notably, scaling the candidate corpus ($K \to \infty$) forces failure unless separation grows as $\log K$. Ultimately, the effective bottleneck is governed by the signal-to-noise ratio $\Delta_{\min}^2 n / \sigma^2$ in the proportional regime.
\subsection{Risk bounds: success vs. failure risk decomposition}
\label{subsec:risk_bounds}

With the retrieval failure probability $\delta$ bounded, we establish the expected prediction risk of the retrieval-guided estimator $\hat{\theta}_{\mathrm{FT}}$. As introduced in Section~\ref{sec:problem_setup}, conditioning on the retrieval outcome $R \in \{0,1\}$ isolates the competing statistical forces of the RAG-FT mechanism.

When retrieval succeeds ($R = 1$), the estimator benefits from pure variance reduction.

\begin{restatable}[\textbf{Risk under Successful Retrieval}]
{proposition}{riskundersucc}
\label{prop:risk_under_successful_retrieval}
Conditioned on the success event $R = 1$, the estimator $\hat{\theta}_{\mathrm{FT}}$ fits
$m+n$ homogeneous samples drawn strictly from the target task $\{\theta_{k^*}\}$. The
estimator is unbiased, and under the good design event $\mathcal{E}_{\mathrm{good}}$ with
$m+n \geq (1+\eta_0)d$ for some $\eta_0 > 0$, its conditional risk satisfies the two-sided
bound
\[
  \sigma^2 + \sigma^2\,\frac{d}{m+n-d-1}
  \;\leq\;
  \mathbb{E}\Bigl[\mathcal{R}(\hat{\theta}_{\mathrm{FT}}) \mid R=1\Bigr]
  \;\leq\;
  \sigma^2 + C\sigma^2\,\frac{d}{m+n},
\]
where the lower bound is the exact inverse-Wishart trace identity
$\mathbb{E}[\mathrm{Tr}((X^\top X)^{-1})] = d/(m+n-d-1)$, valid unconditionally for
$m+n > d+1$, and $C > 0$ is a universal constant governing the upper bound under
$\mathcal{E}_{\mathrm{good}}$.
\end{restatable}

\noindent
When retrieval fails ($R = 0$), the estimator fits a single linear model to a heterogeneous dataset, inducing a high geometric bias.

\begin{restatable}[\textbf{Risk under Retrieval Failure}]
{proposition}{riskunderfailure}
\label{prop:risk_under_failed_retrieval}
Conditioned on the failure event $R=0$, in the high-dimensional proportional
regime where $d\propto(m+n)$ and $m+n>d$, the conditional risk satisfies
\begin{align*}
\mathbb E\!\left[\mathcal R(\hat\theta_{\mathrm{FT}})\mid R=0\right]
&\ge
\sigma^2
+\sigma^2\frac{d}{m+n-d-1}
+\left(\frac{n}{m+n}\right)^2\Delta_{\min}^2,
\\[4pt]
\mathbb E\!\left[\mathcal R(\hat\theta_{\mathrm{FT}})\mid R=0\right]
&\le
\sigma^2
+C\sigma^2\frac{d}{m+n}
+C\left(\frac{n}{m+n}\right)^2
\left(1-\frac{d}{m+n}\right)^{-2}
\Delta_{\max}^2,
\end{align*}
where the lower bound holds unconditionally (for $m+n>d+1$), the
upper bound holds under the good design event $\mathcal E_{\mathrm{good}}$ with
$m+n \geq (1+\eta_0)d$ for some $\eta_0>0$, and $C>0$ is a universal constant.
\end{restatable}

Both propositions isolate the same two statistical forces. Successful retrieval pools $m+n$ homogeneous samples, achieving the exact inverse-Wishart variance floor $d/(m+n-d-1)$. Retrieval failure pools a heterogeneous dataset, inducing a geometric bias scaled by the mixture proportion $n/(m+n)$ and amplified by high-dimensional spectral inflation $(1-d/(m+n))^{-2}$.

Combining these exact conditional bounds with the unconditional failure probability $\delta$ yields the core statistical guarantee of RAG-FT under uniform noise.

\begin{restatable}[\textbf{Total Expected Risk Bound for RAG-FT}]
{theorem}{boundsforRAGFT}
\label{thm:ev_for_ragft}
Under the uniform homoscedastic noise assumption, let the ambient dimension scale proportionally as $d \propto (m + n)$ with $m + n > d$. The expected prediction risk of the retrieval-guided fine-tuning estimator $\hat{\theta}_{\mathrm{FT}}$, evaluated on a fresh sample from the target task, satisfies
\begin{align*}
\mathbb{E}\big[\mathcal{R}(\hat{\theta}_{\mathrm{FT}})\big]
&\ge
\sigma^2
+\sigma^2\frac{d}{m+n-d-1}
+\,\delta\left(\frac{n}{m+n}\right)^2\Delta_{\min}^2,
\\[4pt]
\mathbb{E}\big[\mathcal{R}(\hat{\theta}_{\mathrm{FT}})\big]
&\le
\sigma^2
+C_1\sigma^2\frac{d}{m+n}
+C_2\,\delta\left(\frac{n}{m+n}\right)^2
\left(1-\frac{d}{m+n}\right)^{-2}
\Delta_{\max}^2,
\end{align*}

where $C_1, C_2 > 0$ are universal constants, $\delta$ is the retrieval failure probability bounded in Theorem~\ref{thm:delta_ineq}, and the lower bound is unconditional, with the exact variance floor following from the inverse-Wishart identities underlying Propositions~\ref{prop:risk_under_successful_retrieval} and~\ref{prop:risk_under_failed_retrieval} (valid for m+n>d+1).
\end{restatable}

This bound illustrates how RAG-FT interpolates between target-only fine-tuning and full-corpus aggregation. The bias penalty enters only upon retrieval failure, strictly gated by $\delta$. Because $\delta$ decays exponentially with sufficient task separation, RAG-FT achieves a favorable bias-variance trade-off precisely in this well-separated regime.

\subsection{Strict dominance over training baselines}
\label{subsec:comp_with_baselines}

To contextualize the RAG-FT guarantees, we compare against two standard paradigms: fine-tuning on the target dataset alone, and fine-tuning on the full heterogeneous corpus. These represent the extremes of the bias-variance tradeoff. Target-only training provides an unbiased estimator but suffers high variance $\mathcal{O}(d/m)$. Conversely, full-corpus aggregation minimizes variance to $\mathcal{O}(d/N_{\mathrm{full}})$ but introduces an irreducible bias. Specifically, unless tasks are adversarially anti-correlated, pooling $K$ tasks forces a population bias proportional to the minimum task separation $\Delta_{\min}^2$. The formal risk bounds for these baselines (Propositions~\ref{prop:target-only} and \ref{prop:full-corpus}, alongside the necessary bounded anti-correlation assumption) are deferred to Appendix~\ref{sec:appendix_baselines}.

We establish the exact conditions where RAG-FT strictly dominates both baselines simultaneously by comparing the upper bound of Theorem~\ref{thm:ev_for_ragft} against the lower bounds of the standard estimators.

\noindent
Writing $N_{\mathrm{full}} := Kn+m$, $\Lambda := \dfrac{(m+n-d)^2}{C_2\,
n^2\Delta_{\max}^2}$, and $g(N) := \dfrac{1}{N-d-1} - \dfrac{C_1}{m+n}$ for
$N>d+1$, define the critical thresholds
\[
\delta_{\mathrm{target}} := \Lambda\, \sigma^2 d\, g(m),
\qquad
\delta_{\mathrm{full}} := \Lambda \left( \sigma^2 d\, g(N_{\mathrm{full}})
+ \left(\frac{n}{N_{\mathrm{full}}}\right)^2 (K-1)\big[1-(K-2)\rho\big]
\Delta_{\min}^2 \right).
\]

\begin{restatable}[\textbf{Strict Dominance of RAG-FT}]{theorem}{StrictDominance}
\label{thm:dominance}
Let the expected risks of the standard baselines be bounded as in
Appendix~\ref{sec:appendix_baselines}, the full-corpus bound under Assumption~\ref{assumption:bounded-anticorrelation}. In the
high-dimensional proportional regime $d \propto (m+n)$ with $m+n>d$,
retrieval-guided fine-tuning strictly outperforms both target-only and
full-corpus fine-tuning if
\[
\delta < \min\{\delta_{\mathrm{target}}, \delta_{\mathrm{full}}\}.
\]
Both thresholds must be positive for the claim to be non-vacuous; Remark~\ref{rmk:pos}
gives the precise conditions.
\end{restatable}

Theorem~\ref{thm:dominance} demonstrates that RAG-FT achieves a favorable bias-variance trade-off whenever task separation and sample size are jointly sufficient to drive the retrieval failure probability below these critical thresholds (which is guaranteed by Theorem~\ref{thm:delta_ineq}).

The homoscedastic model is, however, a structural idealization. In practice, the noise variance scales with retrieval rank, and the uniform-weight constraint of linear self-attention cannot accommodate this. Section~\ref{sec:non_uniform_noise} examines the consequences.
\section{Theoretical analysis: under heteroscedastic retrieval noise}
\label{sec:non_uniform_noise}

Section~\ref{sec:uniform_noise} shows that RAG-FT navigates the bias-variance tradeoff well under
homoscedastic noise: the failure probability $\delta$ decays exponentially
in $\Delta_{\min}^2$, and the induced bias is attenuated accordingly. That
result rests on Assumption~\ref{assumption:uniform_noise}, which assigns every candidate dataset the same
noise variance $\sigma^2$ regardless of its distance from the query,
an idealization real corpora violate, since retrieval quality varies
systematically with rank.
This section replaces that idealization with a distance-proportional
noise model and compares how two estimators of the same retrieval process
respond to it.

\subsection{The Distance-Proportional Noise (DPN) model}
\label{subsec:dpn-model}
Let $x_q \in \mathbb{R}^d$ denote the query covariate, $x_q \sim
\mathcal{N}(0,I_d)$. Retrieved samples are indexed by rank $i \in
\{1,\dots,n\}$, with $i=1$ the nearest neighbor under the retrieval metric.
Each retrieved covariate is modeled as
\[
x_{\mathrm{rag}}^{(i)} = x_q + r_i, \qquad r_i \sim \mathcal{N}(0,\delta_i^2I_d),
\]
where $r_i$ is an independent Gaussian offset encoding the geometric
displacement from the query at rank $i$, mutually independent across $i$ and
independent of both $x_q$ and the label noise. Conditional on $x_q$,
$x_{\mathrm{rag}}^{(i)} \sim \mathcal{N}(x_q,\delta_i^2I_d)$; since
$\delta_i^2$ varies with rank, the retrieved covariates are neither i.i.d.
across $i$ nor marginally standard Gaussian.

The label model for a retrieved sample pooled from task $k$ is
\[
y_i = (x_{\mathrm{rag}}^{(i)})^\top \theta_k + \epsilon_i, \qquad
\mathbb{E}[\epsilon_i \mid x_{\mathrm{rag}}^{(i)}] = 0, \qquad
\mathrm{Var}(\epsilon_i \mid x_{\mathrm{rag}}^{(i)}) = \sigma_i^2,
\]
where $\theta_k$ is the parameter of whichever task was retrieved,
coinciding with $\theta_{k^\star}$ under successful retrieval and differing
from it otherwise, as in Section~\ref{sec:uniform_noise}.

\paragraph{Power-law retrieval geometry.} Semantic relevance degrades with
retrieval rank. We model this with a shared power-law scaling of the
covariate offset and the label noise:
\[
\delta_i^2 = \gamma i^q,
\qquad
\sigma_i^2 = \gamma \sigma^2 i^q,
\qquad q>0,
\]
where $\gamma>0$ sets the scale and $q$ governs the rate of noise growth;
coupling the two exponents ensures noisier labels are also,
geometrically farther from the query. This yields independent noise
variables $\varepsilon_i \sim \mathcal{N}(0,\gamma\sigma^2 i^q)$, with
covariance matrix
\[
\Omega := \mathrm{Var}(\varepsilon \mid X) =
\mathrm{diag}\!\left(\gamma\sigma^2 \cdot 1^q,\; \gamma\sigma^2 \cdot 2^q,\;
\ldots,\; \gamma\sigma^2 \cdot n^q\right).
\]
$\Omega$ is diagonal but non-uniform, violating the homoscedasticity
required for OLS to be the optimal linear estimator.

\subsection{OLS risk under Distance-Proportional Noise}

We now turn to the OLS proxy under the DPN model of Section~\ref{subsec:dpn-model}. To
isolate the effect of retrieval volume from the proportional dimensional
scaling used for the OLS baselines of Section~\ref{sec:uniform_noise}, we hold the ambient
dimension $d$, the target sample size $m$, and the noise exponent $q>0$
fixed, and let only the retrieval volume $n$ grow.

\begin{restatable}[\textbf{OLS Risk under DPN}]
{theorem}{OLSriskunderDPN}
\label{thm:ols-dpn}
Under the DPN model, with $d$, $m$, and $q>0$ fixed and $n\to\infty$, the
prediction-risk variance of the OLS proxy,
\[
V_{\mathrm{OLS}}(n) := \mathrm{Tr}\!\left[(X^\top X)^{-1}(X^\top\Omega X)(X^\top X)^{-1}\right],
\]
the Huber--White sandwich covariance trace of Theorem~\ref{thm:white},
where $X \in \mathbb{R}^{(m+n)\times d}$ is the pooled design and
$\Omega \in \mathbb{R}^{(m+n)\times(m+n)}$ is the pooled noise covariance
(diagonal, with target entries $\sigma^2$ and retrieved entries
$\sigma_{\mathrm{rag},i}^2 = \gamma\sigma^2 i^q$), satisfies
\[
\mathbb{E}[V_{\mathrm{OLS}}(n)]
=
\frac{d\,\sigma^2(q+1)^2}{2q+1}
\cdot
\frac1n
\,(1+o(1))
=
\Theta\!\left(\frac dn\right).
\]
for every fixed noise exponent $q>0$, and no divergence occurs at any $q$.
\end{restatable}

\noindent
\textbf{Proof sketch.} The pooled design matrix decomposes as a mean term
$\mu I_d + n\,x_qx_q^\top$ ($\mu = m+s_n = \Theta(n^{q+1})$) plus a
fluctuation of operator norm $O(\sqrt{d}\,n^{q+1/2})$, negligible relative
to the mean for every $q>0$. The mean term is scalar-plus-rank-one and
inverts in closed form via Sherman--Morrison; the same decomposition applied
to the noise-weighted design $X^\top\Omega X$ yields a matching closed form.
Their product has a diagonal trace in the shared eigenbasis that simplifies
to $d\sigma^2(q+1)^2/(2q+1)\cdot n^{-1}$ at leading order, with remaining
approximation error $o(1/n)$ uniformly in $q$.

This absence of divergence mirrors the homoscedastic case of Section~\ref{sec:uniform_noise}:
OLS's inverse-covariance correction adapts to whatever noise structure is
present, so increasing heteroscedasticity changes the constant in $d/n$ but
not the rate. This is the baseline against which
Section~\ref{sec:lsa-dpn} evaluates the literal LSA forward pass.

\subsection{LSA risk under Distance-Proportional Noise}
\label{sec:lsa-dpn}

Assumption~\ref{assumption:oracle_ols_proxy} is not invoked in this section:
retrieval rank in the DPN model is a fixed index, with covariates at each
rank drawn independently of any data-dependent statistic, so no
adaptive-selection step stands between the estimator and the data. This
lets us study the untrained LSA forward pass directly, alongside the
OLS proxy, as two separate, explicitly named estimators of the same DPN
process. The following theorem shows the literal estimator behaves very
differently from the OLS proxy of Theorem~\ref{thm:ols-dpn}.

The literal LSA forward pass produces the prediction $\hat y_q = \frac1n
\sum_{i=1}^n (x_q^\top x_{\mathrm{rag}}^{(i)})\,y_i$, a uniform average over
the retrieved context with no mechanism to weight ranks by their
reliability. Unlike the OLS proxy, this estimator has no matrix inversion
step to correct for the noise structure of the pooled data.
\medskip

\begin{restatable}[\textbf{LSA Risk under DPN}]
{theorem}{LSAriskunderDPN}
\label{thm:lsa-dpn}
Under the DPN model, with $d, m$, and $q > 0$ fixed, $n \to \infty$, and the
retrieved task signal bounded away from zero ($\|\theta_k\| \geq b_0$ for
some fixed constant $b_0 > 0$), the literal LSA forward-pass estimator has
prediction risk
\[
\mathcal{R}_{\mathrm{LSA}}(n) = \sigma^2 + \mathbb{E}[\mathrm{Bias}(x_q)^2] +
\mathbb{E}[\mathrm{Var}(\hat y_q\mid x_q)],
\]
where both terms admit finite-$n$ expressions (Appendix~\ref{subsec:C.2},
exact for the bias, leading-order for the variance) and satisfy, as
$n\to\infty$ for every fixed $q>0$,
\[
\mathbb{E}[\mathrm{Bias}(x_q)^2] = \Theta(n^{2q}),
\qquad
\mathbb{E}[\mathrm{Var}(\hat y_q\mid x_q)] = \Theta(d\cdot n^{2q-1}).
\]
The bias term dominates the variance term for every fixed $q>0$ (their
ratio is $\Theta(n/d)\to\infty$) and diverges regardless of whether the
retrieved task matches the target task, so
\[
\mathcal{R}_{\mathrm{LSA}}(n) = \sigma^2 + \Theta(n^{2q}) + \Theta(d\cdot n^{2q-1})
\to \infty \qquad \text{as } n\to\infty, \text{ for every } q>0.
\]
\end{restatable}

\noindent
\textbf{Proof sketch.}
Conditioning on $x_q$ and writing $x_q = \sqrt\rho\,\omega$ with $\omega$
uniform on the unit sphere, the bias $\mathbb{E}[\hat y_q\mid x_q] -
x_q^\top\theta_{k^\star}$ reduces to a quadratic form in $\rho$; taking
expectation over $\rho \sim \chi^2(d)$ gives an exact closed-form second
moment whose leading term is governed by $(s_n/n)^2\|\theta_k\|^2$, with
$s_n/n = \Theta(n^q)$ and no dependence on $d$ at leading order. The
variance $\mathrm{Var}(\hat y_q\mid x_q)$ combines two independent sources
computed via the same conditioning argument, fluctuation of the retrieved
offset $r_i$ and label noise $\epsilon_i$, giving $\Theta(d\cdot n^{2q-1})$.
Since their ratio is $\Theta(n/d)\to\infty$, the bias term governs the total
risk for every $q>0$.

\subsection{Architectural separation under DPN}

Theorems~\ref{thm:ols-dpn} and~\ref{thm:lsa-dpn} describe the same
retrieval process through two estimators that respond to it very
differently. The OLS proxy's inverse-covariance correction absorbs
whatever heteroscedasticity DPN introduces, so its risk stays
$\Theta(d/n)$ for every $q>0$: the constant in front of $d/n$ grows with
$q$, but the rate does not. The literal LSA forward pass has no such
correction; its prediction is a uniform, unweighted average over retrieved
context, and under DPN this fixed averaging leaves the bias uncalibrated,
diverging as $\Theta(n^{2q})$ for every $q>0$ regardless of retrieval
outcome. Because this divergence carries no leading-order dependence on
$d$, it admits no dimension-dependent ceiling of the kind that bounds risk
in the frozen-weight setting of~\citep{guo2025retrieval}: increasing $d$
does not restore stability. Within this estimator's fixed parameterization
the gap cannot be closed by choice of attention weights, since uniform
averaging is a property of the forward pass we analyze, not a learned
quantity; the failure therefore lies in the absence of a reweighting
mechanism, not in retrieval or the noise model itself. Figure~\ref{fig:app_no-ceiling-reweight} in the
appendix confirms both properties directly.
\section{Empirical validation}
\label{sec:empirics}

The homoscedastic results of Section~\ref{sec:uniform_noise} rest on exact
finite-sample identities, so we target the asymptotic separation asserted by
Theorems~\ref{thm:ols-dpn} and~\ref{thm:lsa-dpn}. Every Monte Carlo point
simulates the generative model of Section~\ref{subsec:dpn-model} and evaluates
each estimator from its defining formula, never from the closed-form risk
expressions plotted as theory curves, so agreement between markers and curves
tests the derivation rather than restating it.

\begin{figure}[htbp]
    \centering
    \includegraphics[width=0.95\textwidth]{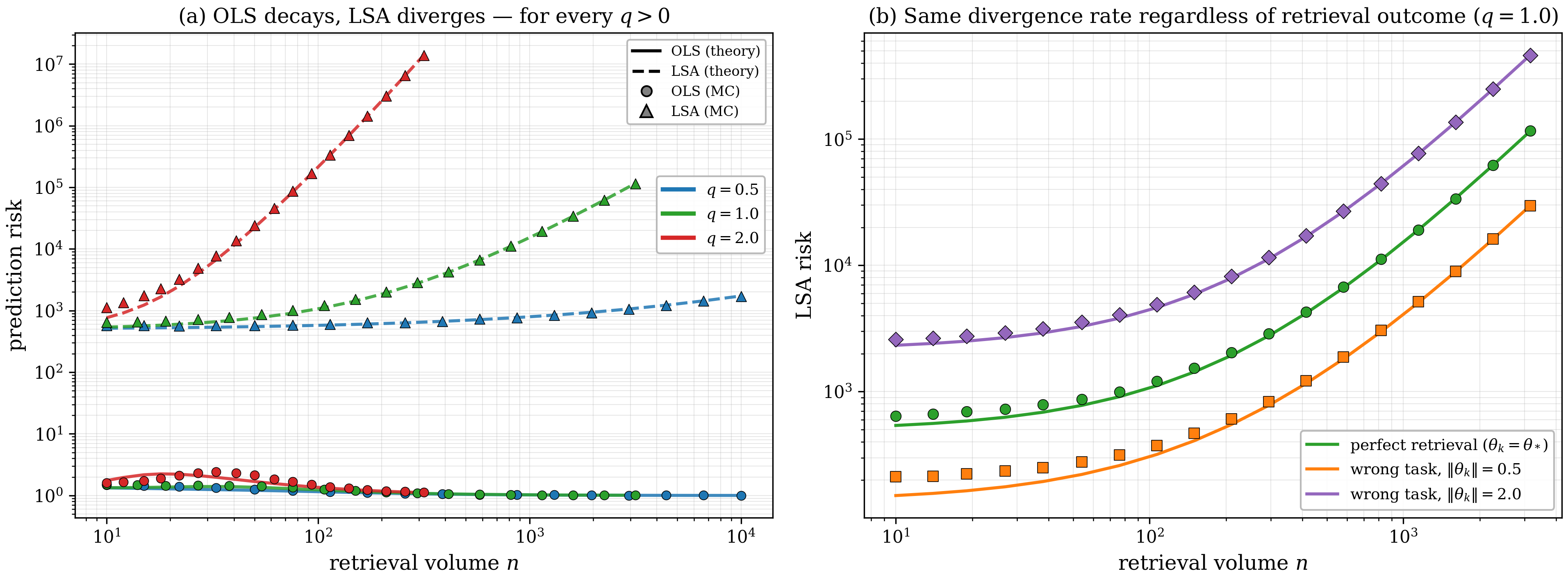}
    \caption{Risk vs.\ retrieval volume $n$ under DPN ($d=20$, $m=50$).
    MC: Monte Carlo; curves: finite-$n$ theory.}
    \label{fig:headline}
\end{figure}

Panel~(a) sweeps retrieval volume $n$ at $d=20$, $m=50$ for
$q\in\{0.5,1,2\}$. OLS follows the predicted $\Theta(d/n)$ decay at every
exponent, while LSA turns upward at $\Theta(n^{2q})$, diverging earlier as $q$
grows. Panel~(b) fixes $q=1$ and varies the retrieved task: exact retrieval
$(\theta_k=\theta_{k^\star})$ diverges at the same rate as two mismatched
tasks, the signature of a bias governed by the aggregation rule rather than by
retrieval error. Appendix~\ref{app:empirical-details} reports two further
predictions of Theorem~\ref{thm:lsa-dpn}: the rate is unchanged across
$d\in\{5,20,50,100\}$, matching a bias term with no leading-order $d$
dependence, and reweighting the same forward pass by reliability restores
stability exactly where uniform averaging fails
(Figure~\ref{fig:app_no-ceiling-reweight}).

\section{Conclusion}

We analyzed Retrieval-Guided Fine-Tuning as a stochastic estimation problem in a
multi-task linear regression framework. Under homoscedastic noise, the retrieval
failure probability decays exponentially in $\Delta_{\min}^2 n/\sigma^2$, and
RAG-FT's risk bound falls strictly below the lower bounds for target-only and
full-corpus training once this failure probability is small enough
(Theorems~\ref{thm:delta_ineq}--\ref{thm:dominance}). Under Distance-Proportional Noise, we separated two estimators
of the same retrieval process: the OLS proxy's risk remains $\Theta(d/n)$
for every noise exponent $q>0$, while the literal LSA forward pass's bias
diverges as $\Theta(n^{2q})$, dominating its variance and growing regardless of
retrieval outcome or ambient dimension (Theorems~\ref{thm:ols-dpn}--\ref{thm:lsa-dpn}). We trace this divergence
to the forward pass's fixed, unweighted aggregation rather than to retrieval
noise itself, and confirm both rate separations by direct simulation
(Section~\ref{sec:empirics}).

\section{Limitations}

Our analysis rests on several simplifying choices. The OLS proxy for the
fine-tuned estimator is an assumption, not something we derive from training
dynamics, so our risk bounds hold only to the extent that this proxy is
accurate. The LSA divergence in Section~\ref{sec:non_uniform_noise} is proven for the fixed, untrained
forward pass; whether a trained attention layer can learn to reweight retrieved
context and close this gap remains an open question. Sections~\ref{sec:uniform_noise} and~\ref{sec:non_uniform_noise} also use
two different asymptotic regimes: one scales dimension with sample size, the
other fixes dimension and grows retrieval volume, and we do not unify them.
More broadly, the model itself is idealized: it assumes Gaussian covariates, a
linear ground-truth relationship, a fixed set of candidate tasks with a single
retrieval step, and a specific noise model for retrieval quality, none of which
we test against real retrieval corpora or trained non-linear models. Finally,
our finite-sample bounds involve unspecified constants and are thus not numerically
tight, and the full-corpus baseline requires a bounded task anti-correlation condition.

\newpage

%

\medskip

{
\small
\bibliographystyle{plainnat}
\bibliography{ref}
}


\newpage
\appendix

\renewcommand{\thetheorem}{\thesection.\arabic{theorem}}
\setcounter{theorem}{0}
\renewcommand{\thelemma}{\thesection.\arabic{lemma}}
\setcounter{lemma}{0}

\section{Technical preliminaries}
\label{sec:A}
\label{app:prelims}

This section collects the standard results from high-dimensional probability and
random matrix theory that are invoked in the proofs of the main theorems.
Statements are given in the forms used in this paper; we refer the reader to the
cited sources for generality and proof.

Several proofs condition on the event that the relevant design matrix is
well-conditioned. For a design matrix $X \in \mathbb{R}^{n\times d}$ with
i.i.d.\ rows drawn from $\mathcal{N}(0,I_d)$, and a fixed constant $C_0>0$,
define the good design event
\[
  \mathcal{E}_{\mathrm{good}} \;:=\; \bigl\{\|(X^\top X)^{-1}\|_{\mathrm{op}} \leq C_0/n\bigr\}.
\]
The design matrix and sample size $n$ instantiating this event vary by context:
in Theorem~\ref{thm:delta_ineq}, $X$ ranges over the per-task design matrices
$X_k$ with $n$ the per-task sample size; in Propositions~\ref{prop:risk_under_successful_retrieval}--\ref{prop:full-corpus},
$X$ is the relevant pooled design matrix, with $n$ replaced by $m+n$, $m$, or
$N$ as appropriate.

\begin{theorem}[Gaussian Tail Bound]
\label{thm:gaussian-tail}
Let $Z \sim \mathcal{N}(0, \sigma^2)$.
For any $t > 0$,
\[
  \mathbb{P}(Z \leq -t) \;\leq\; \exp\!\Bigl(-\frac{t^2}{2\sigma^2}\Bigr).
\]
\end{theorem}

This bound is applied in Step~4 of the proof of Theorem~\ref{thm:delta_ineq}
to control the probability that the linear fluctuation term $L_k$ falls below a
threshold determined by the task separation $\Delta_{\min}$.
See \citet{Vershynin_2018} for a standard reference.

\begin{theorem}[Hanson--Wright Inequality~\citep{rudelson2013hanson}]
\label{thm:hanson-wright}
Let $Z = (Z_1, \ldots, Z_p)^\top \in \mathbb{R}^p$ be a random vector with independent,
mean-zero, sub-Gaussian components satisfying $\|Z_i\|_{\psi_2} \leq K$.
Let $A \in \mathbb{R}^{p \times p}$ be a fixed symmetric matrix.
Then there exists an absolute constant $c > 0$ such that for all $t > 0$,
\[
  \mathbb{P}\bigl(|Z^\top A Z - \mathbb{E}[Z^\top A Z]| \geq t\bigr)
  \;\leq\;
  2\exp\!\left(-c\min\!\left\{\frac{t^2}{K^4 \|A\|_F^2},\,
                               \frac{t}{K^2 \|A\|_{\mathrm{op}}}\right\}\right).
\]
\end{theorem}

In the proof of Lemma~\ref{lemma:quad_fluc_conc}, the quadratic fluctuation term
$Q_k = Z^\top M Z$ is expressed as a quadratic form in the concatenated OLS error vector
$Z \sim \mathcal{N}(0, \Sigma)$, and Theorem~\ref{thm:hanson-wright} is applied with
$A = \Sigma^{1/2} M \Sigma^{1/2}$ to obtain its sub-exponential tail bound.
The Frobenius and operator norms of this symmetrized matrix contribute the
quadratic-fluctuation term to $\nu^2$ and $\alpha$ in Theorem~\ref{thm:delta_ineq}.

\begin{theorem}[Marchenko--Pastur Law]
\label{thm:mp}
Let $X \in \mathbb{R}^{n \times d}$ have i.i.d.\ rows drawn from $\mathcal{N}(0, I_d)$,
and let $\kappa = d/n \in (0,1)$ be fixed as $n \to \infty$.
The empirical spectral distribution of $\frac{1}{n}X^\top X$ converges weakly to the
Marchenko--Pastur distribution with ratio $\kappa$, supported on
$[(1 - \sqrt{\kappa})^2, (1+\sqrt{\kappa})^2]$.
In the proportional regime $d \propto n$, the following non-asymptotic moment estimate
holds: under the good design event $\mathcal{E}_{\mathrm{good}}$, there exists a
constant $C>0$ such that
\[
  \mathbb{E}\bigl[\mathrm{Tr}((X^\top X)^{-1})\;\big|\;\mathcal{E}_{\mathrm{good}}\bigr]
  \;\leq\;
  \frac{Cd}{n},
\]
and the spectral norm of the inverse satisfies
\[
  \mathbb{E}\bigl[\|(X^\top X)^{-1}\|_{\mathrm{op}}\;\big|\;\mathcal{E}_{\mathrm{good}}\bigr]
  \;\leq\;
  \frac{C}{n}\!\left(1 - \frac{d}{n}\right)^{-2}.
\]
\end{theorem}

\noindent
The trace bound is used in the proof of Proposition~\ref{prop:risk_under_successful_retrieval} to
bound the effective variance $\mathbb{E}_x[x^\top(X^\top X)^{-1}x]$ under the pooled
design matrix.
The spectral norm bound appears in the proof of Proposition~\ref{prop:risk_under_failed_retrieval}
as the spectral inflation factor $(1 - d/(m+n))^{-2}$ amplifying the geometric bias
under retrieval failure.
Both bounds follow from inverse Wishart moment calculations;
see \citet{bai2009spectral} for the rigorous random matrix treatment and
\citet{Vershynin_2018} for the non-asymptotic formulation.

\begin{theorem}[Huber--White Covariance Estimator~{\citep{RePEc:ecm:emetrp:v:48:y:1980:i:4:p:817-38}}]
\label{thm:white}
Let $\hat{\theta} = (X^\top X)^{-1} X^\top Y$ be the OLS estimator under the
heteroscedastic model $Y = X\theta^* + \varepsilon$ with
$\mathrm{Var}(\varepsilon_i \mid x_i) = \sigma_i^2$.
The conditional covariance of $\hat{\theta}$ given $X$ is
\[
  \mathrm{Var}(\hat{\theta} \mid X)
  \;=\;
  (X^\top X)^{-1}\!\left(\sum_{i=1}^n \sigma_i^2 x_i x_i^\top\right)\!(X^\top X)^{-1}.
\]
Taking expectation over $X$ with $\mathbb{E}[x_ix_i^\top]=I_d$, and whenever the noise
variances are uniformly bounded by $\sigma^2$,
\[
\mathbb{E}[\mathrm{Tr}(\mathrm{Var}(\hat\theta\mid X))] \leq C\sigma^2\frac{d}{n}.
\]
\end{theorem}

The covariance formula defines $V_{\mathrm{OLS}}(n)$ in
Theorem~\ref{thm:ols-dpn}. Under the DPN model of Section~\ref{subsec:dpn-model},
where $\sigma_i^2$ grows with rank and is not uniformly bounded, the trace
bound above does not apply directly; the DPN-specific trace bound is instead
derived directly in Appendix~\ref{subsec:C.1} via the Sherman--Morrison
decomposition, without appeal to this corollary.

\begin{theorem}[Matrix Bernstein Inequality~\citep{Vershynin_2018}]
\label{thm:matrix-bernstein}
Let $S_1,\ldots,S_n \in \mathbb{R}^{d\times d}$ be independent, symmetric,
mean-zero random matrices with $\max_i \|S_i\|_{\psi_1} \leq K$. Let
$v := \bigl\|\sum_{i=1}^n \mathbb{E}[S_i^2]\bigr\|_{\mathrm{op}}$ denote the
matrix variance proxy. Then there exists an absolute constant $c>0$ such that,
with probability at least $1 - d^{-c}$,
\[
  \Bigl\|\sum_{i=1}^n S_i\Bigr\|_{\mathrm{op}} \;=\; O\bigl(\sqrt{v} + K\bigr).
\]
\end{theorem}

This inequality is applied in the proofs of Theorem~\ref{thm:ols-dpn} and
Theorem~\ref{thm:lsa-dpn} to bound the operator norm of the
fluctuation terms arising from the rank-dependent covariance structure of the
DPN model, where each $S_i$ is a centered outer product $r_ir_i^\top - \delta_i^2I_d$
(or its noise-weighted analogue), and the variance proxy $v$ is computed via the
Gaussian fourth-moment identity $\mathbb{E}[\|z\|^2zz^\top] = (d+2)I_d$ for
$z \sim \mathcal{N}(0,I_d)$.

\begin{theorem}[Sherman--Morrison Identity]
\label{thm:sherman-morrison}
Let $A \in \mathbb{R}^{d\times d}$ be invertible and $u,v \in \mathbb{R}^d$. If
$1 + v^\top A^{-1}u \neq 0$, then $A + uv^\top$ is invertible and
\[
  (A+uv^\top)^{-1} \;=\; A^{-1} - \frac{A^{-1}uv^\top A^{-1}}{1+v^\top A^{-1}u}.
\]
\end{theorem}

This identity is applied in the proofs of Theorem~\ref{thm:ols-dpn} and
Theorem~\ref{thm:lsa-dpn} to invert the conditional-mean design
matrix $\bar\Sigma = \mu I_d + n\,x_qx_q^\top$ in closed form.

\begin{theorem}[Exact Inverse Wishart Trace Identity]
\label{thm:inverse-wishart-exact}
Let $X \in \mathbb{R}^{N \times d}$ have i.i.d.\ rows drawn from $\mathcal{N}(0, I_d)$,
with $N > d+1$. Then $X^\top X$ follows a Wishart distribution
$\mathcal{W}_d(N, I_d)$, and
\[
  \mathbb{E}\bigl[\mathrm{Tr}((X^\top X)^{-1})\bigr] \;=\; \frac{d}{N-d-1}.
\]
\end{theorem}

This is the exact first moment of the inverse Wishart distribution and holds
without any high-probability conditioning, unlike the Marchenko--Pastur trace
bound of Theorem~\ref{thm:mp}, which is a conditional upper bound valid under
the good design event. The exact identity is used to establish the variance lower
bound in Propositions~\ref{prop:risk_under_successful_retrieval}
and~\ref{prop:risk_under_failed_retrieval}, and their analogues for the target-only
and full-corpus baselines in Section~\ref{subsec:comp_with_baselines}.

\begin{lemma}[Concentration of Empirical Covariance]
\label{lem:wishart-concentration}
Under Gaussian design $X \in \mathbb{R}^{n\times d}$ with i.i.d.\ rows drawn
from $\mathcal{N}(0,I_d)$, in the proportional regime $n \geq c_0 d$ for a fixed
constant $c_0>0$, there exists $c'>0$ such that
\[
  \mathbb{P}(\mathcal{E}_{\mathrm{good}}^c) \;\leq\; 3\exp(-c'd).
\]
\end{lemma}

This is a standard consequence of non-asymptotic random matrix concentration;
see \citet{Vershynin_2018}, Chapter~4. Applying this bound to each of the $K$ per-task design matrices individually and taking
a union bound gives $\mathbb{P}(\mathcal{E}_{\mathrm{good}}^c) \leq 3K\exp(-c'd)$ for the
intersection event used in Theorem~\ref{thm:delta_ineq}, representing the irreducible
high-dimensional bottleneck that persists even with infinite task data.

\renewcommand{\thelemma}{\arabic{lemma}}
\setcounter{lemma}{2}

\section{Proofs for Section 4 (homoscedastic regime)}
\label{sec:B}

\subsection{Auxiliary lemmas (Lemmas 1--4)}
\label{subsec:B.1}

\lemmaone*

\begin{proof}
In vector form the linear model reads $y_k = X_k \theta_k + \epsilon_k$,
with $\epsilon_k \sim \mathcal{N}(0, \sigma^2 I_n)$ independent of $X_k$.
Substituting into the OLS closed form,
\begin{align*}
  \hat{\theta}_k
  &= (X_k^\top X_k)^{-1} X_k^\top y_k \\
  &= (X_k^\top X_k)^{-1} X_k^\top (X_k \theta_k + \epsilon_k) \\
  &= \theta_k + (X_k^\top X_k)^{-1} X_k^\top \epsilon_k,
\end{align*}
so the estimation error is $\hat{\theta}_k - \theta_k =
(X_k^\top X_k)^{-1} X_k^\top \epsilon_k$.
Conditional on $X_k$, this is a linear map applied to
$\epsilon_k \sim \mathcal{N}(0,\sigma^2 I_n)$, hence Gaussian with mean
zero and covariance
\begin{align*}
  \operatorname{Var}\!\left(\hat{\theta}_k - \theta_k \mid X_k\right)
  &= (X_k^\top X_k)^{-1} X_k^\top \cdot \sigma^2 I_n \cdot X_k (X_k^\top X_k)^{-1} \\
  &= \sigma^2 (X_k^\top X_k)^{-1} X_k^\top X_k (X_k^\top X_k)^{-1} \\
  &= \sigma^2 (X_k^\top X_k)^{-1}.
\end{align*}
An identical argument applied to the $m$-sample query dataset completes
the proof.
\end{proof}

\lemmatwo*

\begin{proof}
Write each empirical estimator as $\hat{\theta}_j = \theta_j + \xi_j$,
where $\xi_j$ is the OLS estimation error. Then
\begin{align*}
  \hat{\theta}_k - \hat{\theta}_q
  &= (\theta_k + \xi_k) - (\theta_{k^*} + \xi_q)
   = (\theta_k - \theta_{k^*}) + (\xi_k - \xi_q),
\end{align*}
and expanding the squared norm,
\begin{equation*}
  \left\lVert\hat{\theta}_k - \hat{\theta}_q\right\rVert_2^2
  = \lVert\theta_k - \theta_{k^*}\rVert_2^2
  + 2\langle \theta_k - \theta_{k^*},\, \xi_k - \xi_q \rangle
  + \lVert\xi_k - \xi_q\rVert_2^2.
\end{equation*}
For the baseline term, the deterministic signal $\theta_{k^*} -
\theta_{k^*}$ vanishes identically, giving
\begin{equation*}
  \left\lVert\hat{\theta}_{k^*} - \hat{\theta}_q\right\rVert_2^2
  = \lVert\xi_{k^*} - \xi_q\rVert_2^2.
\end{equation*}
Subtracting and collecting terms,
\begin{align*}
  T_k
  &= \left\lVert\hat{\theta}_k - \hat{\theta}_q\right\rVert_2^2
   - \left\lVert\hat{\theta}_{k^*} - \hat{\theta}_q\right\rVert_2^2 \\[4pt]
  &= \lVert\theta_k - \theta_{k^*}\rVert_2^2
   + 2\langle \theta_k - \theta_{k^*},\, \xi_k - \xi_q \rangle
   + \lVert\xi_k - \xi_q\rVert_2^2
   - \lVert\xi_{k^*} - \xi_q\rVert_2^2 \\[4pt]
  &= \lVert\theta_k - \theta_{k^*}\rVert_2^2 + L_k + Q_k,
\end{align*}
which is the claimed decomposition.
\end{proof}

\begin{lemma}[\textbf{Distribution of the Linear Fluctuation Term}]
\label{lemma:linear_fluc_term}
For any incorrect task $k \neq k^\ast$, define the linear fluctuation term induced by estimation noise as
\[
L_k := 2 \langle \theta_k - \theta_{k^\ast}, \xi_k - \xi_q \rangle,
\]
where $\xi$ denotes the respective OLS estimation errors. Conditional on the design matrices $(X_k, X_q)$, $L_k$ is a mean-zero Gaussian random variable with variance:
\[
\mathrm{Var}(L_k \mid X_k, X_q) = 4 (\theta_k - \theta_{k^\ast})^\top \Sigma_{k,q} (\theta_k - \theta_{k^\ast}),
\]
where $\Sigma_{k,q} = \sigma^2 (X_k^\top X_k)^{-1} + \sigma^2 (X_q^\top X_q)^{-1}$.
\end{lemma}

\begin{proof}
Lemma~\ref{lemma:finite_sample_dist_of_ols_estimation_noise} gives the conditional distributions $\xi_k \sim \mathcal{N}(0,
\sigma^2(X_k^\top X_k)^{-1})$ and $\xi_q \sim \mathcal{N}(0,
\sigma^2(X_q^\top X_q)^{-1})$.
Because task datasets are generated independently, $\xi_k$ and $\xi_q$
are independent, so
\begin{align*}
  \operatorname{Var}(\xi_k - \xi_q \mid X_k, X_q)
  &= \operatorname{Var}(\xi_k \mid X_k) + \operatorname{Var}(\xi_q \mid X_q) \\
  &= \sigma^2 (X_k^\top X_k)^{-1} + \sigma^2 (X_q^\top X_q)^{-1} \\
  &= \Sigma_{k,q}.
\end{align*}
Set $v_k := \theta_k - \theta_{k^*}$.
Then $L_k = 2 v_k^\top (\xi_k - \xi_q)$ is a linear functional of a
mean-zero Gaussian vector, hence itself mean-zero Gaussian.
Its conditional variance is
\begin{align*}
  \operatorname{Var}(L_k \mid X_k, X_q)
  &= 4\, v_k^\top\, \operatorname{Var}(\xi_k - \xi_q \mid X_k, X_q)\, v_k \\
  &= 4(\theta_k - \theta_{k^*})^\top \Sigma_{k,q}\, (\theta_k - \theta_{k^*}),
\end{align*}
which is the stated result.
\end{proof}

\begin{lemma}[\textbf{Quadratic Fluctuation Concentration}]
\label{lemma:quad_fluc_conc}
Fix an incorrect task index $k \neq k^\ast$ and define the quadratic fluctuation term as
\[
Q_k := \|\xi_k - \xi_q\|_2^2 - \|\xi_{k^\ast} - \xi_q\|_2^2.
\]
Let $\mathcal{X}$ denote the collection of all corresponding design matrices. The conditional expectation of this fluctuation is:
\[
\mathbb{E}[Q_k \mid \mathcal{X}] = \sigma^2 \mathrm{Tr}\left((X_k^\top X_k)^{-1}\right) - \sigma^2 \mathrm{Tr}\left((X_{k^\ast}^\top X_{k^\ast})^{-1}\right).
\]
Furthermore, $Q_k$ exhibits sub-exponential concentration around its mean. There exists an absolute constant $c > 0$ such that for all $t > 0$:
\[
\mathbb{P}\left(\left|Q_k - \mathbb{E}[Q_k \mid \mathcal{X}]\right| \ge t \mid \mathcal{X}\right)
\le 2 \exp\left(
- c \min \left(
\frac{t^2}{\|\Sigma^{1/2} M \Sigma^{1/2}\|_F^2},
\frac{t}{\|\Sigma^{1/2} M \Sigma^{1/2}\|_{\mathrm{op}}}
\right)
\right),
\]
where $M \in \mathbb{R}^{3d \times 3d}$ is a fixed symmetric block matrix governing the quadratic form, and $\Sigma$ is the conditional block-diagonal covariance matrix of the joint error vector.
\end{lemma}

\begin{proof}
By Lemma~\ref{lemma:finite_sample_dist_of_ols_estimation_noise}, the estimation errors satisfy $\xi_j \sim
\mathcal{N}(0, \Sigma_j)$ with $\Sigma_j = \sigma^2(X_j^\top X_j)^{-1}$
for $j \in \{k, k^*, q\}$, and are mutually independent across datasets.
Define the concatenated error $Z := [\xi_k^\top, \xi_{k^*}^\top,
\xi_q^\top]^\top \in \mathbb{R}^{3d}$; conditional on $\mathcal{X}$,
$Z \sim \mathcal{N}(0, \Sigma)$ with $\Sigma =
\operatorname{diag}(\Sigma_k, \Sigma_{k^*}, \Sigma_q)$.

Expanding the squared norms in $Q_k$ and cancelling the common
$\xi_q^\top \xi_q$ terms,
\begin{align*}
  Q_k
  &= \left(\xi_k^\top \xi_k - 2\xi_k^\top \xi_q + \xi_q^\top \xi_q\right)
   - \left(\xi_{k^*}^\top \xi_{k^*} - 2\xi_{k^*}^\top \xi_q + \xi_q^\top \xi_q\right) \\
  &= \xi_k^\top \xi_k - \xi_{k^*}^\top \xi_{k^*}
   - 2\xi_k^\top \xi_q + 2\xi_{k^*}^\top \xi_q \\
  &= Z^\top M Z,
\end{align*}
where $M \in \mathbb{R}^{3d \times 3d}$ is the symmetric block matrix
\begin{equation*}
  M \;=\;
  \begin{pmatrix}
    I_d & 0 & -I_d \\
    0 & -I_d & I_d \\
    -I_d & I_d & 0
  \end{pmatrix}.
\end{equation*}
For a mean-zero Gaussian vector, $\mathbb{E}[Z^\top M Z \mid \mathcal{X}]
= \operatorname{Tr}(M\Sigma)$.
Executing the block-matrix product,
\begin{align*}
  \operatorname{Tr}(M\Sigma)
  &= \operatorname{Tr}
     \begin{pmatrix}
       \Sigma_k & 0 & -\Sigma_q \\
       0 & -\Sigma_{k^*} & \Sigma_q \\
       -\Sigma_k & \Sigma_{k^*} & 0
     \end{pmatrix} \\[4pt]
  &= \operatorname{Tr}(\Sigma_k) - \operatorname{Tr}(\Sigma_{k^*}) \\[4pt]
  &= \sigma^2\operatorname{Tr}\!\left((X_k^\top X_k)^{-1}\right)
   - \sigma^2\operatorname{Tr}\!\left((X_{k^*}^\top X_{k^*})^{-1}\right),
\end{align*}
establishing the stated conditional expectation.
The tail bound follows by applying the Hanson--Wright inequality
(Theorem~\ref{thm:hanson-wright}) to the Gaussian quadratic form $Z^\top M Z$, with the
symmetrized matrix $\Sigma^{1/2} M \Sigma^{1/2}$ controlling both the
Frobenius and operator norm terms.
\end{proof}

\subsection{Proof of Theorem 1}

\RetrievalFailureProb*

\begin{proof}
Lemma~\ref{lemma:signal_fluc_decomp} establishes that the retrieval decision statistic for any
incorrect task $k \neq k^{*}$ decomposes exactly as
\begin{equation*}
  T_k \;=\; \lVert\theta_k - \theta_{k^{*}}\rVert_2^2 + L_k + Q_k,
\end{equation*}
where $L_k$ is the linear fluctuation term characterized in Lemma~\ref{lemma:linear_fluc_term} and
$Q_k$ is the quadratic fluctuation term characterized in Lemma~\ref{lemma:quad_fluc_conc}.
The nearest-estimator retrieval rule selects the correct task if and only
if $T_k > 0$ for every $k \neq k^{*}$, so the failure event is
\begin{equation*}
  \mathcal{E} \;=\; \bigl\{\exists\, k \neq k^{*} : T_k \leq 0\bigr\}.
\end{equation*}

\textit{Step 1: Reduction to a single task.}
Fix an arbitrary incorrect task $k \neq k^{*}$.
Because $\lVert\theta_k - \theta_{k^{*}}\rVert_2^2 \geq \Delta_{\min}^2$
by definition, the event $\{T_k \leq 0\}$ requires
\begin{equation*}
  L_k + Q_k \;\leq\; -\lVert\theta_k - \theta_{k^{*}}\rVert_2^2
            \;\leq\; -\Delta_{\min}^2.
\end{equation*}

\textit{Step 2: Centering the quadratic fluctuation.}
Define the centered quadratic fluctuation $\widetilde{Q}_k := Q_k -
\mathbb{E}[Q_k \mid \mathcal{X}]$.
Substituting yields
\begin{equation*}
  L_k + \widetilde{Q}_k
  \;\leq\; -\Delta_{\min}^2 - \mathbb{E}[Q_k \mid \mathcal{X}].
\end{equation*}
The assumption $\Delta_{\min}^2 > 2\lvert\mathbb{E}[Q_k \mid \mathcal{X}]\rvert$
implies $-\mathbb{E}[Q_k \mid \mathcal{X}] \leq \Delta_{\min}^2/2$,
so the right-hand side is at most $-\Delta_{\min}^2/2$:
\begin{equation*}
  \{T_k \leq 0\}
  \;\subseteq\;
  \left\{ L_k + \widetilde{Q}_k \leq -\frac{\Delta_{\min}^2}{2} \right\}.
\end{equation*}

\textit{Step 3: Threshold splitting.}
Splitting at $-\Delta_{\min}^2/4$ and applying the union bound,
\begin{align*}
  \mathbb{P}\!\left(T_k \leq 0 \;\middle|\; \mathcal{X}\right)
  &\;\leq\;
  \mathbb{P}\!\left(
    L_k + \widetilde{Q}_k \leq -\frac{\Delta_{\min}^2}{2}
    \;\middle|\; \mathcal{X}
  \right) \\[4pt]
  &\;\leq\;
  \mathbb{P}\!\left(L_k \leq -\frac{\Delta_{\min}^2}{4}
    \;\middle|\; \mathcal{X}\right)
  \;+\;
  \mathbb{P}\!\left(\widetilde{Q}_k \leq -\frac{\Delta_{\min}^2}{4}
    \;\middle|\; \mathcal{X}\right).
\end{align*}

\textit{Step 4: Bounding the linear fluctuation tail.}
Lemma~\ref{lemma:linear_fluc_term} establishes that $L_k$, conditional on $\mathcal{X}$, is mean-zero
Gaussian with variance $\operatorname{Var}(L_k \mid \mathcal{X})$.
The Gaussian tail inequality (Theorem~\ref{thm:gaussian-tail}) at threshold
$t = \Delta_{\min}^2/4$ gives
\begin{align*}
  \mathbb{P}\!\left(L_k \leq -\frac{\Delta_{\min}^2}{4}
    \;\middle|\; \mathcal{X}\right)
  &\;\leq\; \exp\!\left(
    -\frac{(\Delta_{\min}^2/4)^2}{2\operatorname{Var}(L_k \mid \mathcal{X})}
  \right) \\[4pt]
  &\;=\; \exp\!\left(
    -\frac{\Delta_{\min}^4}{32\operatorname{Var}(L_k \mid \mathcal{X})}
  \right).
\end{align*}

\textit{Step 5: Bounding the quadratic fluctuation tail.}
Lemma~\ref{lemma:quad_fluc_conc} expresses $\widetilde{Q}_k = Z^\top M Z - \mathbb{E}[Z^\top M Z \mid \mathcal{X}]$,
where $Z \sim \mathcal{N}(0,\Sigma)$ is the concatenated estimation error
and $M \in \mathbb{R}^{3d \times 3d}$ is a fixed symmetric block matrix.
Applying the Hanson--Wright inequality at threshold
$t = \Delta_{\min}^2/4$,
\begin{align*}
  \mathbb{P}\!\left(\widetilde{Q}_k \leq -\frac{\Delta_{\min}^2}{4}
    \;\middle|\; \mathcal{X}\right)
  &\;\leq\;
  2\exp\!\left(
    -c\min\!\left\{
      \frac{(\Delta_{\min}^2/4)^2}{\lVert\Sigma^{1/2}M\Sigma^{1/2}\rVert_F^2},\;
      \frac{\Delta_{\min}^2/4}{\lVert\Sigma^{1/2}M\Sigma^{1/2}\rVert_{\mathrm{op}}}
    \right\}
  \right) \\[4pt]
  &\;=\;
  2\exp\!\left(
    -c\min\!\left\{
      \frac{\Delta_{\min}^4}{16\lVert\Sigma^{1/2}M\Sigma^{1/2}\rVert_F^2},\;
      \frac{\Delta_{\min}^2}{4\lVert\Sigma^{1/2}M\Sigma^{1/2}\rVert_{\mathrm{op}}}
    \right\}
  \right).
\end{align*}

\textit{Step 6a: Combining and absorbing constants.}
Define the composite variance parameters
\begin{equation*}
  \nu^{2} \;:=\; \operatorname{Var}(L_k \mid \mathcal{X})
           \;+\; \bigl\lVert\Sigma^{1/2}M\Sigma^{1/2}\bigr\rVert_{F}^{2},
  \qquad
  \alpha  \;:=\; \max\!\left(
                   \sqrt{\operatorname{Var}(L_k \mid \mathcal{X})},\;
                   \bigl\lVert\Sigma^{1/2}M\Sigma^{1/2}\bigr\rVert_{\mathrm{op}}
                 \right).
\end{equation*}
Absorbing the numerical prefactors into a single universal constant $C > 0$,
the combined single-task conditional bound is
\begin{equation*}
  \mathbb{P}\!\left(T_k \leq 0 \;\middle|\; \mathcal{X}\right)
  \;\leq\;
  3\exp\!\left(
    -C\min\!\left\{
      \frac{\Delta_{\min}^4}{\nu^{2}},\;
      \frac{\Delta_{\min}^2}{\alpha}
    \right\}
  \right).
\end{equation*}

\textit{Step 6b: Deterministic bounds under the good design event.}
Under $\mathcal{E}_{\mathrm{good}}$, $\|(X_k^\top X_k)^{-1}\|_{\mathrm{op}}
\le C_0/n$ for every $k$. By Lemma~\ref{lemma:linear_fluc_term},
\[
\mathrm{Var}(L_k\mid\mathcal X)
=
4(\theta_k-\theta_{k^*})^\top\Sigma_{k,q}(\theta_k-\theta_{k^*})
\le
\frac{8C_0\sigma^2\Delta_{\max}^2}{n}
=: \bar\nu_1^2,
\]
since $\Sigma_{k,q}\preceq (2C_0\sigma^2/n)I_d$. Moreover,
\[
\|\Sigma\|_{\mathrm{op}}
\le \frac{2C_0\sigma^2}{n},
\qquad
\Sigma=\mathrm{diag}(\Sigma_k,\Sigma_{k^*},\Sigma_q),
\]
and therefore
\begin{align*}
\|\Sigma^{1/2}M\Sigma^{1/2}\|_F^2
&\le
\|\Sigma\|_{\mathrm{op}}^2\|M\|_F^2
=
O\!\left(\frac{d\sigma^4}{n^2}\right),\\
\|\Sigma^{1/2}M\Sigma^{1/2}\|_{\mathrm{op}}
&\le
\|\Sigma\|_{\mathrm{op}}\|M\|_{\mathrm{op}}
=
O\!\left(\frac{\sigma^2}{n}\right).
\end{align*}
Hence
\[
\nu^2 \le \bar\nu^2
:=
O\!\left(
\frac{\sigma^2\Delta_{\max}^2}{n}
+
\frac{d\sigma^4}{n^2}
\right),
\qquad
\alpha \le \bar\alpha
:=
O\!\left(
\frac{\sigma\Delta_{\max}}{\sqrt n}
+
\frac{\sigma^2}{n}
\right).
\]

Similarly,
\[
|\mathbb E[Q_k\mid\mathcal X]|
\le
\sigma^2\operatorname{Tr}\!\bigl((X_k^\top X_k)^{-1}\bigr)
+
\sigma^2\operatorname{Tr}\!\bigl((X_{k^*}^\top X_{k^*})^{-1}\bigr)
\le
\frac{2C_0\sigma^2 d}{n}.
\]
Therefore $\Delta_{\min}^2 \ge C_1\sigma^2 d/n$ with $C_1=4C_0$ implies
\[
\Delta_{\min}^2 > 2|\mathbb E[Q_k\mid\mathcal X]|,
\]
recovering the condition used in Step~2.

\textit{Step 7: Union bound over all incorrect tasks.}
The preceding bound is symmetric in $k$ under Assumption~\ref{assumption:uniform_noise} (all tasks
share the same noise level $\sigma^2$ and sample size $n$).
Applying a union bound over the $K-1$ incorrect candidate tasks,
\begin{equation*}
  \mathbb{P}\!\left(\mathcal{E} \;\middle|\; \mathcal{X}\right)
  \;\leq\;
  3(K-1)\exp\!\left(
    -C\min\!\left\{
      \frac{\Delta_{\min}^4}{\nu^{2}},\;
      \frac{\Delta_{\min}^2}{\alpha}
    \right\}
  \right).
\end{equation*}

\textit{Step 8: Removing the conditioning via the good design event.}
The bounds in Steps~4--6 rely on the conditional variance parameters
$\nu^2$ and $\alpha$, which are well-defined only when the empirical
covariance matrices $(X_k^\top X_k)^{-1}$ are well-conditioned.
Define the good design event
\begin{equation*}
  \mathcal{E}_{\mathrm{good}}
  \;:=\;
  \left\{
    \bigl\lVert(X_k^\top X_k)^{-1}\bigr\rVert_{\mathrm{op}} \leq \frac{C_0}{n}
    \;\text{ for all } k \in [K]
  \right\}
\end{equation*}
for a fixed constant $C_0 > 0$.

By Lemma~\ref{lem:wishart-concentration}, each individual event
$\mathcal{E}_k := \{\|(X_k^\top X_k)^{-1}\|_{\mathrm{op}} \leq C_0/n\}$ satisfies
$\mathbb{P}(\mathcal{E}_k^c) \leq 3\exp(-c'd)$. Since
$\mathcal{E}_{\mathrm{good}} = \bigcap_{k=1}^K \mathcal{E}_k$, a union bound over the
$K$ tasks gives
\[
    \mathbb{P}(\mathcal{E}_{\mathrm{good}}^c)
    = \mathbb{P}\Bigl(\bigcup_{k=1}^K \mathcal{E}_k^c\Bigr)
    \leq \sum_{k=1}^K \mathbb{P}(\mathcal{E}_k^c)
    \leq 3K\exp(-c'd).
\]
The law of total probability then gives
\begin{align*}
  \delta
  \;=\; \mathbb{P}(\mathcal{E})
  &\;\leq\;
  \mathbb{P}\!\left(\mathcal{E} \;\middle|\; \mathcal{E}_{\mathrm{good}}\right)
  \mathbb{P}(\mathcal{E}_{\mathrm{good}})
  \;+\;
  \mathbb{P}\!\left(\mathcal{E}_{\mathrm{good}}^c\right) \\[4pt]
  &\;\leq\;
  \mathbb{P}\!\left(\mathcal{E} \;\middle|\; \mathcal{E}_{\mathrm{good}}\right)
  \;+\;
  \mathbb{P}\!\left(\mathcal{E}_{\mathrm{good}}^c\right) \\[4pt]
  &\;\leq\;
  3(K-1)\exp\!\left(
    -C\min\!\left\{
      \frac{\Delta_{\min}^4}{\nu^{2}},\;
      \frac{\Delta_{\min}^2}{\alpha}
    \right\}
  \right)
  \;+\; 3K\exp(-c' d),
\end{align*}
where the last inequality substitutes the conditional bound from Step~7
and the good design concentration estimate.
This is the claimed bound, completing the proof.
\end{proof}

\subsection{Proofs of conditional risk bounds (Propositions 1 and 2)}

\riskundersucc*

\begin{proof}
Conditioned on $\mathcal{R} = 1$, the retrieved corpus consists
entirely of samples from the target task $\theta_{k^*}$, so the pooled
dataset of $m + n$ observations satisfies the well-specified linear model
$y_i = x_i^\top \theta_{k^*} + \epsilon_i$ with
$\epsilon_i \sim \mathcal{N}(0, \sigma^2)$ i.i.d.
The OLS estimator is therefore unbiased:
$\mathbb{E}[\hat{\theta}_{\mathrm{FT}} \mid X, R=1] = \theta_{k^*}$.

For a fresh test point $(x, y)$ with $y = x^\top \theta_{k^*} + \epsilon$,
the prediction error decomposes as
\begin{align*}
  y - x^\top \hat{\theta}_{\mathrm{FT}}
  &= \epsilon - x^\top(\hat{\theta}_{\mathrm{FT}} - \theta_{k^*}),
\end{align*}
and since the test noise $\epsilon$ is independent of $\hat{\theta}_{\mathrm{FT}}$,
\begin{align*}
  \mathbb{E}\!\left[\mathcal{R}(\hat{\theta}_{\mathrm{FT}}) \;\middle|\; R=1\right]
  &= \sigma^2 + \mathbb{E}\!\left[
       x^\top \operatorname{Var}(\hat{\theta}_{\mathrm{FT}} \mid X,\, R=1)\, x
     \right].
\end{align*}

From Lemma~\ref{lemma:finite_sample_dist_of_ols_estimation_noise}, the conditional covariance of the OLS estimator on the
pooled design matrix $X \in \mathbb{R}^{(m+n)\times d}$ is
$\operatorname{Var}(\hat{\theta}_{\mathrm{FT}} \mid X, R=1)
= \sigma^2(X^\top X)^{-1}$.
Taking the expectation over a fresh test covariate
$x \sim \mathcal{N}(0, I_d)$, independent of $X$,
\begin{align*}
  \mathbb{E}_x\!\left[
    x^\top (X^\top X)^{-1} x
  \right]
  &= \operatorname{Tr}\!\left((X^\top X)^{-1}\right).
\end{align*}

It remains to bound $\mathbb{E}[\operatorname{Tr}((X^\top X)^{-1})]$ under
the good design event $\mathcal{E}_{\mathrm{good}}$.
Under Gaussian design in the proportional regime $m+n \ge (1+\eta_0)d$ for the fixed constant $\eta_0>0$ defining the good design event, the Marchenko--Pastur moment estimate (Theorem~\ref{thm:mp}) gives
\begin{equation*}
  \mathbb{E}\!\left[\operatorname{Tr}\!\left((X^\top X)^{-1}\right)
    \;\middle|\; \mathcal{E}_{\mathrm{good}}\right]
  \;\leq\; \frac{Cd}{m+n}
\end{equation*}
for a universal constant $C > 0$.
Combining,
\begin{equation*}
  \mathbb{E}\!\left[\mathcal{R}(\hat{\theta}_{\mathrm{FT}}) \;\middle|\; R=1\right]
  \;\leq\; \sigma^2 + C\sigma^2\frac{d}{m+n},
\end{equation*}
which is the stated bound.

\medskip
\textit{Lower bound.}
The conditional covariance established above is exact:
$\mathrm{Var}(\hat{\theta}_{\mathrm{FT}} \mid X, R=1) = \sigma^2(X^\top X)^{-1}$, with
$X \in \mathbb{R}^{(m+n)\times d}$ the pooled design matrix under Gaussian design.
Since $X^\top X \sim \mathcal{W}_d(m+n, I_d)$, Theorem~\ref{thm:inverse-wishart-exact}
gives the exact identity
\[
  \mathbb{E}\bigl[\mathrm{Tr}((X^\top X)^{-1})\bigr] = \frac{d}{m+n-d-1},
\]
valid for $m+n > d+1$, without conditioning on the good design event. Combining with
the irreducible noise term,
\[
  \mathbb{E}\Bigl[\mathcal{R}(\hat{\theta}_{\mathrm{FT}}) \mid R=1\Bigr]
  \;=\; \sigma^2 + \sigma^2\,\mathbb{E}\bigl[\mathrm{Tr}((X^\top X)^{-1})\bigr]
  \;=\; \sigma^2 + \sigma^2\frac{d}{m+n-d-1},
\]
which is the stated lower bound. Together with the upper bound derived above, this
establishes the two-sided result.

\end{proof}

\riskunderfailure*

\begin{proof}

\textbf{Setup.}
Conditioned on $R=0$, an incorrect task $k\neq k^*$ is retrieved.
The pooled dataset consists of $m$ target samples satisfying
$y_i = x_i^\top\theta_{k^*}+\varepsilon_i$ and $n$ retrieved samples satisfying
$y_i^{\mathrm{ret}} = (x_i^{\mathrm{ret}})^\top\theta_k + \varepsilon_i^{\mathrm{ret}}$,
all with homoscedastic noise $\varepsilon\sim\mathcal{N}(0,\sigma^2)$.

\textbf{Population target and unbiasedness.}
Let $X\in\mathbb{R}^{(m+n)\times d}$ be the pooled design matrix and $y\in\mathbb{R}^{m+n}$
the response vector. Since both sub-populations have isotropic design $x\sim\mathcal{N}(0,I_d)$,
the population minimizer of the expected squared loss on the pooled mixture is
\[
  \bar\theta := \frac{m}{m+n}\theta_{k^*} + \frac{n}{m+n}\theta_k.
\]
For a fixed realization of $X$, the OLS estimator $\hat\theta_{\mathrm{FT}} = (X^\top X)^{-1}X^\top y$ is generally \emph{not} conditionally unbiased for $\bar\theta$; only its expectation over the random design coincides with $\bar\theta$, as we verify below.
The bias with respect to the true target is
\[
  \bar\theta - \theta_{k^*} = \frac{n}{m+n}(\theta_k-\theta_{k^*}).
\]

\textbf{Bias-variance decomposition.}
For a fresh test point $x\sim\mathcal{N}(0,I_d)$ independent of $X$ and the training noise,
write the prediction error as
\begin{align*}
  y - x^\top\hat\theta_{\mathrm{FT}}
  &= \varepsilon_{\mathrm{test}}
     - x^\top(\hat\theta_{\mathrm{FT}} - \bar\theta)
     - x^\top(\bar\theta - \theta_{k^*}).
\end{align*}
While $\mathbb{E}[\hat\theta_{\mathrm{FT}} \mid X, R=0] = \bar\theta + \Delta(X)$ for a
generally nonzero deterministic function
$\Delta(X) := \big[\tfrac{m}{m+n}I_d -
(X^\top X)^{-1}X_{\mathrm{target}}^\top X_{\mathrm{target}}\big](\theta_k-\theta_{k^*})$,
its expectation over the random design vanishes exactly:
$X_{\mathrm{target}}^\top X_{\mathrm{target}}$ and $X_{\mathrm{retrieved}}^\top
X_{\mathrm{retrieved}}$ are independent Wishart matrices from rotation-invariant Gaussian
designs, so $\mathbb{E}[(X^\top X)^{-1}X_{\mathrm{target}}^\top X_{\mathrm{target}}]$
commutes with every orthogonal matrix and must therefore be a scalar multiple of $I_d$;
matching traces via the matrix-variate Beta identity
$\mathbb{E}[\mathrm{Tr}((X^\top X)^{-1}X_{\mathrm{target}}^\top X_{\mathrm{target}})] =
\tfrac{m}{m+n}d$ pins this scalar at $m/(m+n)$ exactly, so $\mathbb{E}_X[\Delta(X)] = 0$.
Consequently, taking expectation over both $X$ and the training noise, the cross term
between $\hat\theta_{\mathrm{FT}}-\bar\theta$ and the fixed vector $\bar\theta-\theta_{k^*}$
vanishes exactly, giving
\begin{equation}
  \mathbb{E}[\mathcal{R}(\hat\theta_{\mathrm{FT}})\mid R=0]
  \ge \sigma^2
  + \underbrace{\mathbb{E}\!\left[\operatorname{Tr}\!\left(\operatorname{Var}
    (\hat\theta_{\mathrm{FT}}\mid X,\,R=0)\right)\right]}_{\text{Variance}}
  + \underbrace{\left\|\bar\theta - \theta_{k^*}\right\|^2}_{\text{Squared bias}}.
  \label{eq:prop2-decomp}
\end{equation}

\textbf{Squared bias term.}
Since $\bar\theta-\theta_{k^*} = \frac{n}{m+n}(\theta_k-\theta_{k^*})$,
\[
  \|\bar\theta - \theta_{k^*}\|^2
  = \left(\frac{n}{m+n}\right)^{\!2}\|\theta_k-\theta_{k^*}\|^2.
\]
For the upper bound, $\|\theta_k-\theta_{k^*}\|\leq\Delta_{\max}$ gives
$\|\bar\theta-\theta_{k^*}\|^2 \leq (n/(m+n))^2\Delta_{\max}^2$.
For the lower bound, $\|\theta_k-\theta_{k^*}\|\geq\Delta_{\min}$ (since $k\neq k^*$) gives
$\|\bar\theta-\theta_{k^*}\|^2 \geq (n/(m+n))^2\Delta_{\min}^2$.

\textbf{Variance term: upper bound.}
Let $r_i := y_i - x_i^\top \bar\theta$ denote the $i$-th pooled residual under $\bar\theta$. Under
homoscedastic noise, the true conditional variance is exactly
$\mathrm{Var}(\hat\theta_{\mathrm{FT}}\mid X, R=0) = \sigma^2(X^\top X)^{-1}$; the additional
fluctuation from model misspecification is captured separately via $\Delta(X)$ above. To bound
the two together in a single expression, we use the Huber--White sandwich quantity
\[
(X^\top X)^{-1}\Big(\sum_{i=1}^{m+n} r_i^2 x_ix_i^\top\Big)(X^\top X)^{-1},
\]
which provides a valid upper bound on $\mathrm{Tr}(\mathrm{Var}(\hat\theta_{\mathrm{FT}}\mid X,R=0)) +
\|\Delta(X)\|^2$ jointly; this is what the subsequent decomposition into noise and
misspecification terms bounds.
By construction, each residual satisfies $r_i = x_i^\top(\theta_{\mathrm{true},i}-\bar\theta)+\varepsilon_i$,
where $\theta_{\mathrm{true},i}=\theta_{k^*}$ for target rows and $\theta_k$ for retrieved rows.
Applying the AM inequality $(a+b)^2\leq 2a^2+2b^2$,
\begin{equation}
  r_i^2
  \;\leq\; 2\bigl(x_i^\top(\theta_{\mathrm{true},i}-\bar\theta)\bigr)^2 + 2\varepsilon_i^2.
  \label{eq:prop2-ri-bound}
\end{equation}
Substituting \eqref{eq:prop2-ri-bound} and taking the trace of the sandwich,
\begin{align}
  \operatorname{Tr}\!\left(\operatorname{Var}(\hat\theta_{\mathrm{FT}}\mid X,\,R=0)\right)
  &\;\leq\;
  \underbrace{2\sum_{i=1}^{m+n}\varepsilon_i^2\cdot x_i^\top(X^\top X)^{-2}x_i}_{\text{Noise term}}
  \notag \\
  &\quad
  +\underbrace{2\sum_{i=1}^{m+n}\bigl(x_i^\top(\theta_{\mathrm{true},i}-\bar\theta)\bigr)^2\cdot
  x_i^\top(X^\top X)^{-2}x_i}_{\text{Misspecification term}}.
  \label{eq:prop2-trace-split}
\end{align}
For the noise term, since $\varepsilon_i$ is independent of $X$ with $\mathbb{E}[\varepsilon_i^2]=\sigma^2$:
\[
  \mathbb{E}\!\left[2\sum_i\varepsilon_i^2\cdot x_i^\top(X^\top X)^{-2}x_i\right]
  = 2\sigma^2\,\mathbb{E}\!\left[\operatorname{Tr}\!\left((X^\top X)^{-1}\right)\right]
  = \frac{2\sigma^2 d}{m+n-d-1},
\]
using the exact inverse-Wishart identity from Theorem~\ref{thm:inverse-wishart-exact}.
For the misspecification term, write $u := \theta_k-\theta_{k^*}$; the shifts satisfy
$\|\theta_{\mathrm{true},i}-\bar\theta\|\leq\Delta_{\max}$ for all rows, since both
$\|(n/(m+n))u\|$ and $\|(m/(m+n))u\|$ are at most $\|u\|$.
Applying the Cauchy--Schwarz inequality and the Marchenko--Pastur spectral norm bound
(Theorem~\ref{thm:mp})
$\mathbb{E}[\|(X^\top X)^{-1}\|_{\mathrm{op}}\mid\mathcal{E}_{\mathrm{good}}]\leq C(m+n)^{-1}(1-d/(m+n))^{-2}$,
the misspecification term is bounded by
\[
  C\left(\frac{n}{m+n}\right)^{\!2}\!\left(1-\frac{d}{m+n}\right)^{\!-2}\!\Delta_{\max}^2,
\]
where the constant $C$ absorbs the factor of $2$ and the proportionality constants.
Combining with the upper bound on the squared bias and the exact variance floor
$2\sigma^2 d/(m+n-d-1)\leq C\sigma^2 d/(m+n)$ under $\mathcal{E}_{\mathrm{good}}$,
\[
  \mathbb{E}[\mathcal{R}(\hat\theta_{\mathrm{FT}})\mid R=0]
  \;\leq\; \sigma^2 + C\sigma^2\frac{d}{m+n}
  + C\!\left(\frac{n}{m+n}\right)^{\!2}\!\left(1-\frac{d}{m+n}\right)^{\!-2}\!\Delta_{\max}^2.
\]

\textbf{Lower bound.}
The variance floor follows from the same exact inverse-Wishart identity applied to the
pooled Gaussian design $X^\top X\sim\mathcal{W}_d(m+n,I_d)$:
\[
  \mathbb{E}[\operatorname{Tr}((X^\top X)^{-1})] = \frac{d}{m+n-d-1},\quad m+n>d+1,
\]
unconditionally and without conditioning on $\mathcal{E}_{\mathrm{good}}$.
For the bias lower bound, since $\bar\theta-\theta_{k^*}$ is a deterministic vector
and $x\sim\mathcal{N}(0,I_d)$ is independent of $X$,
\[
  \mathbb{E}_x\!\left[\bigl|x^\top(\bar\theta-\theta_{k^*})\bigr|^2\right]
  = \|\bar\theta-\theta_{k^*}\|^2
  = \left(\frac{n}{m+n}\right)^{\!2}\|\theta_k-\theta_{k^*}\|^2
  \;\geq\; \left(\frac{n}{m+n}\right)^{\!2}\Delta_{\min}^2.
\]
This is an exact identity, not a concentration result, and requires no conditioning on
$\mathcal{E}_{\mathrm{good}}$. Combining the variance floor with the bias lower bound,
\[
  \mathbb{E}[\mathcal{R}(\hat\theta_{\mathrm{FT}})\mid R=0]
  \;\geq\;
  \sigma^2 + \sigma^2\frac{d}{m+n-d-1}
  + \left(\frac{n}{m+n}\right)^{\!2}\!\Delta_{\min}^2.
\]
This omits a further non-negative contribution $\mathbb{E}_X[\|\Delta(X)\|^2]\ge0$
from the fluctuation term $\Delta(X)$ introduced above: its exact value requires
matrix-variate Beta second moments beyond the elementary Wishart identities used
here, but rotation invariance pins its scale at $O(\Delta_{\max}^2\,mn/(m+n)^3)$,
dominated by the retained bias term throughout the regime $d\propto(m+n)$
considered here, so omitting it costs nothing qualitative. Together with the
upper bound, this establishes the two-sided result.
\end{proof}

\subsection{Proof for Theorem 2}

\boundsforRAGFT*

\begin{proof}
The retrieval event $\mathcal{R} \in \{0,1\}$ is a Bernoulli random variable
with $\mathbb{P}(\mathcal{R} = 0) = \delta$ and $\mathbb{P}(\mathcal{R} = 1) = 1 - \delta$.
Conditioning on this event and applying the law of total expectation,
\begin{align*}
  \mathbb{E}\!\left[\mathcal{R}(\hat{\theta}_{\mathrm{FT}})\right]
  &= (1 - \delta)\,\mathbb{E}\!\left[\mathcal{R}(\hat{\theta}_{\mathrm{FT}}) \;\middle|\; R=1\right]
   + \delta\,\mathbb{E}\!\left[\mathcal{R}(\hat{\theta}_{\mathrm{FT}}) \;\middle|\; R=0\right].
\end{align*}

Substituting the conditional bounds established in Propositions~\ref{prop:risk_under_successful_retrieval} and~\ref{prop:risk_under_failed_retrieval}
and using $(1 - \delta) \leq 1$,
\begin{align*}
  \mathbb{E}\!\left[\mathcal{R}(\hat{\theta}_{\mathrm{FT}})\right]
  &\leq
  (1-\delta)\!\left(\sigma^2 + C_1\sigma^2\frac{d}{m+n}\right) \\
  &\quad
  + \delta\!\left(
      \sigma^2
      + C_1\sigma^2\frac{d}{m+n}
      + C_2\!\left(\frac{n}{m+n}\right)^{\!2}\!\left(1 - \frac{d}{m+n}\right)^{\!-2}\!\Delta_{\max}^2
    \right) \\[6pt]
  &= \sigma^2 + C_1\sigma^2\frac{d}{m+n}
   \\
  &\quad
   + C_2\delta\!\left(\frac{n}{m+n}\right)^{\!2}\!\left(1 - \frac{d}{m+n}\right)^{\!-2}\!\Delta_{\max}^2,
\end{align*}
where the last equality collapses the shared $\sigma^2 + C_1\sigma^2 d/(m+n)$
terms across both conditioning events, since they carry identical coefficients
that sum to one under the mixture $(1-\delta) + \delta = 1$.
 
\medskip
\noindent\textit{Lower bound.}
Applying the same law of total expectation to the lower bounds of
Propositions~\ref{prop:risk_under_successful_retrieval} and~\ref{prop:risk_under_failed_retrieval},
\[
  \mathbb{E}\Bigl[\mathcal{R}(\hat{\theta}_{\mathrm{FT}})\Bigr]
  \;=\;
  (1-\delta)\,\mathbb{E}\Bigl[\mathcal{R}(\hat{\theta}_{\mathrm{FT}}) \mid R=1\Bigr]
  + \delta\,\mathbb{E}\Bigl[\mathcal{R}(\hat{\theta}_{\mathrm{FT}}) \mid R=0\Bigr]
\]
\[
  \;\geq\;
  (1-\delta)\left(\sigma^2 + \sigma^2\frac{d}{m+n-d-1}\right)
  + \delta\left(\sigma^2 + \sigma^2\frac{d}{m+n-d-1}
  + \left(\frac{n}{m+n}\right)^{\!2}\Delta_{\min}^2\right)
\]
\[
  \;=\;
  \sigma^2 + \sigma^2\frac{d}{m+n-d-1}
  + \,\delta\left(\frac{n}{m+n}\right)^{\!2}\Delta_{\min}^2,
\]
where the shared terms collapse exactly as in the upper bound derivation, since
$(1-\delta) + \delta = 1$. This establishes the lower bound stated in
Theorem~\ref{thm:ev_for_ragft}, completing the two-sided result.
\end{proof}

\begin{remark}
The expected bias penalty in Theorem~\ref{thm:ev_for_ragft} is gated by $\delta$ and therefore
decays exponentially in the signal-to-noise ratio $\Delta_{\min}^2 n / \sigma^2$
once the task separation exceeds the estimation noise floor.
In the well-separated regime $\Delta_{\min}^2 \gg \sigma^2 \sqrt{d/n}$,
the bias term is negligible and the bound reduces to $\sigma^2 + \mathcal{O}(d/(m+n))$,
recovering the same rate as Proposition~\ref{prop:risk_under_successful_retrieval}.
Conversely, in the degenerate regime $\delta \to 1$, arising when the task
corpus is dense and $\Delta_{\min}$ is small, the bias penalty dominates and
the bound degrades to the full failure-mode risk of Proposition~\ref{prop:risk_under_failed_retrieval}.
The factor $\left(1 - d/(m+n)\right)^{-2}$ quantifies the spectral
amplification of the geometric bias by near-singular empirical covariances
in the proportional high-dimensional regime; it diverges as $d/(m+n) \to 1$,
confirming that the estimator requires $m + n \gg d$ to remain stable.
\end{remark}

\subsection{Baseline risk bounds (target-only and full-corpus)}
\label{sec:appendix_baselines}

In this section, we formally state and prove the risk bounds for the two standard fine-tuning baselines discussed in Section~\ref{subsec:comp_with_baselines}: target-only fine-tuning and full-corpus fine-tuning.

\begin{proposition}[\textbf{Risk of Target-Only Fine-Tuning}]
\label{prop:target-only}
Let $\hat{\theta}_{\mathrm{target}}$ denote the OLS estimator trained exclusively on
the query dataset $\mathcal{D}'$ of $m$ i.i.d.\ samples from the target task
$\theta_{k^*}$, with $m > d+1$. The expected prediction risk satisfies
\begin{equation*}
    \sigma^2 + \sigma^2 \frac{d}{m-d-1}
    \;\leq\;
    \mathbb{E}\bigl[\mathcal{R}(\hat{\theta}_{\mathrm{target}})\bigr]
    \;\leq\;
    \sigma^2 + C_1 \sigma^2 \frac{d}{m},
    \label{eq:target-only}
\end{equation*}
where the lower bound is the exact inverse-Wishart identity and the upper bound holds
under the good design event $\mathcal{E}_{\mathrm{good}}$ with $m \geq (1+\eta_0)d$
for some $\eta_0 > 0$.
\end{proposition}\textit{Proof of Proposition~\ref{prop:target-only}.}
The target-only estimator $\hat{\theta}_{\mathrm{target}}$ is OLS on $m$ homogeneous samples
from $\theta_{k^*}$, which is exactly the setting of Proposition~\ref{prop:risk_under_successful_retrieval} with
$n = 0$. The two-sided bound follows immediately, provided $m > d+1$. \qed

\bigskip

The full-corpus bias bound below requires ruling out exact cancellation among opposing task deviations.

\begin{assumption}[\textbf{Bounded Task Anti-Correlation}]
\label{assumption:bounded-anticorrelation}

For $K\ge3$, there exists $\rho \in [0,\,1/(K-2))$ such that the task deviations
$\delta_k := \theta_k-\theta_{k^*}$ satisfy
\[
\langle \delta_i,\delta_j\rangle
\ge -\rho\,\Delta_{\min}^2,
\qquad
i\neq j,\quad i,j\in [K]\setminus\{k^*\}.
\]

No such assumption is required when $K=2$.
\end{assumption}

\begin{remark}
Assumption~\ref{assumption:bounded-anticorrelation} rules out only the degenerate case in
which task deviations are strongly and systematically opposed. It is automatically
satisfied whenever the $\{\delta_k\}$ are not adversarially anti-correlated (for instance,
whenever all pairwise inner products are non-negative, the boundary case $\rho = 0$
applies). Without some assumption of this form, the full-corpus population bias can vanish
by exact cancellation: e.g.\ for $K=2$ with $\delta_1 = -\delta_2$, the population target
$\bar\theta_{\mathrm{full}}$ coincides with $\theta_{k^*}$ even though
$\|\delta_1\| = \|\delta_2\| = \Delta_{\min} > 0$. Assumption~\ref{assumption:bounded-anticorrelation} is the weakest condition of this
kind that restores a non-trivial lower bound for $K\geq 3$, requiring $\rho<1/(K-2)$;
for $K=2$ there are no cross terms and no anti-correlation assumption is needed.
\end{remark}

\begin{proposition}[\textbf{Risk of Full-Corpus Fine-Tuning}]
\label{prop:full-corpus}
Let $\hat{\theta}_{\mathrm{full}}$ denote the OLS estimator trained on the pooled
dataset $\mathcal{D}' \cup \bigcup_{k=1}^{K} \mathcal{D}_k$ of $N_{\mathrm{full}} = m + Kn$ samples.
Define the population minimizer of the full-corpus mixture loss as
\begin{equation*}
    \bar{\theta}_{\mathrm{full}} := \frac{m\,\theta_{k^*} + n\sum_{k=1}^{K} \theta_k}{m + Kn},
\end{equation*}
so that $\hat\theta_{\mathrm{full}}$ is unbiased for $\bar\theta_{\mathrm{full}}$ in expectation over the random design (not conditionally, for a fixed design realization).
Under Assumption~\ref{assumption:bounded-anticorrelation} and in the
high-dimensional proportional regime $d \propto N_{\mathrm{full}}$ with $N_{\mathrm{full}} > d$, the expected
prediction risk satisfies
\begin{align*}
    &\sigma^2 + \sigma^2 \frac{d}{N_{\mathrm{full}}-d-1}
    + \left(\frac{n}{N_{\mathrm{full}}}\right)^{\!2} (K-1)\bigl[1-(K-2)\rho\bigr]\,\Delta_{\min}^2
    \\
    &\;\leq\;
    \mathbb{E}\bigl[\mathcal{R}(\hat{\theta}_{\mathrm{full}})\bigr]
    \\
    \mathbb{E}\bigl[\mathcal{R}(\hat{\theta}_{\mathrm{full}})\bigr]
    &\;\leq\;
    \sigma^2 + C_1 \sigma^2 \frac{d}{N_{\mathrm{full}}}
    \\
    &\quad
    + C_2 \left(\frac{Kn}{N_{\mathrm{full}}}\right)^{\!2} \left(1-\frac{d}{N_{\mathrm{full}}}\right)^{-2} \Delta_{\max}^2,
\end{align*}
where the lower bound holds unconditionally (for $N_{\mathrm{full}}>d+1$), the
upper bound holds under $\mathcal{E}_{\mathrm{good}}$ with $N_{\mathrm{full}} \geq (1+\eta_0)d$
for some $\eta_0>0$, and $C_2 > 0$ is a universal constant.
\end{proposition}

\begin{proof}
The pooled dataset consists of $m$ target samples satisfying $y_i = x_i^\top \theta_{k^*} +
\varepsilon_i$ and, for each $k \in [K]$, $n$ task samples satisfying $y_i^{(k)} =
(x_i^{(k)})^\top \theta_k + \varepsilon_i^{(k)}$, with homoscedastic noise $\sigma^2$
throughout. Writing $N := N_{\mathrm{full}} = m+Kn$ throughout this proof, let
$X_{\mathrm{full}} \in \mathbb{R}^{N \times d}$ and $Y_{\mathrm{full}} \in
\mathbb{R}^N$ denote the stacked design matrix and response vector.
\medskip

\textbf{Population target and unbiasedness.}
Since the design is isotropic across all tasks, the parameter minimizing the expected
full-corpus squared loss is
\[
    \bar{\theta}_{\mathrm{full}}
    = \frac{m\,\theta_{k^*} + n\sum_{k=1}^{K}\theta_k}{m + Kn}.
\]
For a fixed realization of $X_{\mathrm{full}}$, the OLS estimator is generally \emph{not} conditionally unbiased for $\bar\theta_{\mathrm{full}}$; only its expectation over the random design equals $\bar\theta_{\mathrm{full}}$, as we verify below. The bias with respect to the true target is
\[
    \bar{\theta}_{\mathrm{full}} - \theta_{k^*}
    = \frac{n}{m+Kn}\sum_{k=1}^{K}(\theta_k - \theta_{k^*})
    = \frac{n}{m+Kn}\sum_{k=1}^{K}\delta_k.
\]

\textbf{Bias-variance decomposition.}
For a fresh test point $x \sim \mathcal{N}(0, I_d)$ independent of $X_{\mathrm{full}}$, the
prediction error decomposes as
\[
    y - x^\top\hat{\theta}_{\mathrm{full}}
    = \varepsilon
    - x^\top(\hat{\theta}_{\mathrm{full}} - \bar{\theta}_{\mathrm{full}})
    - x^\top(\bar{\theta}_{\mathrm{full}} - \theta_{k^*}).
\]
While $\mathbb{E}[\hat\theta_{\mathrm{full}} \mid X_{\mathrm{full}}] = \bar\theta_{\mathrm{full}}
+ \Delta(X_{\mathrm{full}})$ for a generally nonzero deterministic function
\[
\Delta(X_{\mathrm{full}}) := \sum_{k=0}^{K}\left[(X_{\mathrm{full}}^\top X_{\mathrm{full}})^{-1}X_k^\top X_k
- \frac{n_k}{m+Kn}I_d\right]\theta_k,
\]
(writing $X_0:=X_{\mathrm{target}}$, $\theta_0:=\theta_{k^*}$, $n_0:=m$, and $n_k:=n$ for
$k\geq1$), its expectation over the random design vanishes exactly by the same
rotation-invariance argument used in Proposition~\ref{prop:risk_under_failed_retrieval}:
each block $X_k^\top X_k$ is an independent Wishart matrix from a rotation-invariant
Gaussian design, so $\mathbb{E}[(X_{\mathrm{full}}^\top X_{\mathrm{full}})^{-1}X_k^\top X_k]$
commutes with every orthogonal matrix and equals $\tfrac{n_k}{m+Kn}I_d$ exactly, giving
$\mathbb{E}_{X_{\mathrm{full}}}[\Delta(X_{\mathrm{full}})]=0$. Consequently, taking
expectation over both $X_{\mathrm{full}}$ and the training noise, the cross term between
$\hat\theta_{\mathrm{full}}-\bar\theta_{\mathrm{full}}$ and the fixed vector
$\bar\theta_{\mathrm{full}}-\theta_{k^*}$ vanishes exactly, giving
\[
    \mathbb{E}\bigl[\mathcal{R}(\hat{\theta}_{\mathrm{full}})\bigr]
    \;\ge\; \sigma^2
    + \mathbb{E}\bigl[\mathrm{Tr}\bigl(\mathrm{Var}(\hat{\theta}_{\mathrm{full}}
    \mid X_{\mathrm{full}})\bigr)\bigr]
    + \bigl\|\bar{\theta}_{\mathrm{full}} - \theta_{k^*}\bigr\|^2.
\]

\textbf{Variance term.}
Conditional on the design, the randomness from the label noise gives the usual $\sigma^2(X_{\mathrm{full}}^\top X_{\mathrm{full}})^{-1}$ covariance around the empirical conditional mean. The additional mismatch bias is handled separately through the population mixture offset plus the design fluctuation term $\Delta(X_{\mathrm{full}})$, mirroring the structure in Proposition~\ref{prop:risk_under_failed_retrieval}. Thus, the variance from the label noise alone is exactly
$\mathrm{Var}(\hat{\theta}_{\mathrm{full}} \mid X_{\mathrm{full}}) =
\sigma^2(X_{\mathrm{full}}^\top X_{\mathrm{full}})^{-1}$.
By Theorem~\ref{thm:inverse-wishart-exact} applied to the $N$-row Gaussian design matrix,
\[
    \mathbb{E}\bigl[\mathrm{Tr}(X_{\mathrm{full}}^\top X_{\mathrm{full}})^{-1}\bigr]
    = \frac{d}{N - d - 1},
\]
which is exact and unconditional, giving the variance floor $\sigma^2 d/(N-d-1)$. Under
$\mathcal{E}_{\mathrm{good}}$ with $N \geq (1+\eta_0)d$, the standard
inverse-Wishart upper bound yields $\mathbb{E}[\mathrm{Tr}(X_{\mathrm{full}}^\top
X_{\mathrm{full}})^{-1} \mid \mathcal{E}_{\mathrm{good}}] \leq Cd/N$.

\textbf{Bias term: upper bound.}
By the triangle inequality $\bigl\|\sum_{k=1}^K \delta_k\bigr\| \leq K\Delta_{\max}$, so
\[
    \bigl\|\bar{\theta}_{\mathrm{full}} - \theta_{k^*}\bigr\|^2
    \leq \left(\frac{Kn}{m+Kn}\right)^{\!2}\Delta_{\max}^2.
\]
In the proportional regime the Marchenko--Pastur spectral norm bound
(Theorem~\ref{thm:mp}) gives
$\mathbb{E}[\|(X_{\mathrm{full}}^\top X_{\mathrm{full}})^{-1}\|_{\mathrm{op}} \mid
\mathcal{E}_{\mathrm{good}}] \leq C(m+Kn)^{-1}(1 - d/N)^{-2}$,
amplifying the bias by $(1-d/N)^{-2}$ and establishing the upper bound.

\textbf{Bias term: lower bound.}
Expanding the squared norm of the sum exactly via the polarization identity,
\[
    \Bigl\|\sum_{k=1}^{K}\delta_k\Bigr\|^2
    = \sum_{k=1}^{K}\|\delta_k\|^2
    + \sum_{i \neq j}\langle \delta_i, \delta_j \rangle.
\]
Since $\delta_{k^*} = 0$, the first sum is really a sum over $k \neq k^*$, and since
$\|\delta_k\| \geq \Delta_{\min}$ for all $k \neq k^*$, it satisfies
$\sum_{k\ne k^\ast} \|\delta_k\|^2 \ge (K-1)\Delta_{\min}^2$. Likewise, any cross term
involving $k^*$ vanishes identically, so only the $(K-1)(K-2)$ cross terms with both
indices in $[K]\setminus\{k^*\}$ can contribute. By
Assumption~\ref{assumption:bounded-anticorrelation}, each such term satisfies
$\langle \delta_i, \delta_j \rangle \geq -\rho\,\Delta_{\min}^2$, so
\[
    \sum_{i \neq j}\langle \delta_i, \delta_j \rangle
    \;\geq\; -\rho\,\Delta_{\min}^2\, (K-1)(K-2).
\]
Combining,
\[
    \Bigl\|\sum_{k=1}^{K}\delta_k\Bigr\|^2
    \;\geq\; (K-1)\Delta_{\min}^2 - \rho\,(K-1)(K-2)\Delta_{\min}^2
    \;=\; (K-1)\bigl[1-\rho(K-2)\bigr]\Delta_{\min}^2,
\]
which is strictly positive since $\rho < 1/(K-2)$ for $K \geq 3$ (for $K=2$ there is
no cross term, and the bound reduces exactly to $\Delta_{\min}^2$). Therefore the
population squared bias satisfies
\[
    \bigl\|\bar{\theta}_{\mathrm{full}} - \theta_{k^*}\bigr\|^2
    \;\geq\; \left(\frac{n}{m+Kn}\right)^{\!2} (K-1)\bigl[1-\rho(K-2)\bigr]\Delta_{\min}^2.
\]

Since $\bar{\theta}_{\mathrm{full}}-\theta_{k^*}$ is a fixed vector depending only on the
deterministic task parameters $\{\theta_k\}$, and $x\sim N(0,I_d)$ is independent of
$X_{\mathrm{full}}$, the identity
$\mathbb{E}_x[(x^\top(\bar{\theta}_{\mathrm{full}}-\theta_{k^*}))^2]=
\|\bar{\theta}_{\mathrm{full}}-\theta_{k^*}\|^2$ holds exactly, with no conditioning
on $\mathcal{E}_{\mathrm{good}}$. 
Combining variance and bias terms establishes the upper bound; combining the
variance floor with the bias lower bound establishes the stated lower bound,
up to a further non-negative contribution
$\mathbb{E}_{X_{\mathrm{full}}}[\|\Delta(X_{\mathrm{full}})\|^2]$ from the
fluctuation term introduced above, whose exact value again requires
matrix-variate second moments beyond the elementary identities used here but
whose scale, $O(\Delta_{\max}^2\,mKn/N^3)$, is dominated by the retained bias
term in the regime $d\propto N$ considered throughout, so omitting it costs
nothing qualitative.

\end{proof}
\subsection{Proof of strict dominance (Theorem 3)}

\StrictDominance*

\begin{proof}
The proof solves each of the two strict risk inequalities
$\mathbb{E}[\mathcal{R}(\hat{\theta}_{\mathrm{FT}})] < \mathbb{E}[\mathcal{R}(\hat{\theta}_{\mathrm{target}})]$
and $\mathbb{E}[\mathcal{R}(\hat{\theta}_{\mathrm{FT}})] < \mathbb{E}[\mathcal{R}(\hat{\theta}_{\mathrm{full}})]$
for $\delta$, comparing the upper bound on RAG-FT risk (Theorem~\ref{thm:ev_for_ragft})
against the \emph{lower} bound on each baseline (Propositions~\ref{prop:target-only}
and~\ref{prop:full-corpus}), and intersects the resulting admissible sets.

\bigskip
\textbf{Simplifying the spike factor.}
Throughout, we use the algebraic identity
\begin{align}
    \left(\frac{n}{m+n}\right)^{\!2}\left(1-\frac{d}{m+n}\right)^{\!-2}
    &= \left(\frac{n}{m+n}\right)^{\!2}\left(\frac{m+n}{m+n-d}\right)^{\!2}
    = \frac{n^2}{(m+n-d)^2}. \label{eq:spike-simplify}
\end{align}

\textbf{Dominance over Target-Only Fine-Tuning.}
From Theorem~\ref{thm:ev_for_ragft} and Proposition~\ref{prop:target-only}, the
dominance condition $\mathbb{E}[\mathcal{R}(\hat{\theta}_{\mathrm{FT}})] < \mathbb{E}[\mathcal{R}(\hat{\theta}_{\mathrm{target}})]$ requires
\begin{align}
    \sigma^2 + C_1\sigma^2\frac{d}{m+n}
    + C_2\delta\left(\frac{n}{m+n}\right)^{\!2}\left(1-\frac{d}{m+n}\right)^{\!-2}\Delta_{\max}^2
    &< \sigma^2 + \frac{\sigma^2 d}{m-d-1}. \label{eq:target-ineq-raw}
\end{align}
Cancelling $\sigma^2$ from both sides and applying \eqref{eq:spike-simplify},
\eqref{eq:target-ineq-raw} becomes
\begin{align}
    C_1\sigma^2\frac{d}{m+n} + C_2\delta\,\frac{n^2}{(m+n-d)^2}\,\Delta_{\max}^2
    &< \frac{\sigma^2 d}{m-d-1}. \label{eq:target-cancel-sigma}
\end{align}
Moving the variance term to the right-hand side of \eqref{eq:target-cancel-sigma} and
factoring out $\sigma^2 d$,
\begin{align}
    C_2\delta\,\frac{n^2}{(m+n-d)^2}\,\Delta_{\max}^2
    &< \sigma^2 d\left[\frac{1}{m-d-1} - \frac{C_1}{m+n}\right]. \label{eq:target-isolate}
\end{align}
Dividing both sides of \eqref{eq:target-isolate} by $C_2 n^2 \Delta_{\max}^2 /(m+n-d)^2 > 0$
and rearranging,
\begin{align}
    \delta
    &< \frac{(m+n-d)^2}{C_2 n^2 \Delta_{\max}^2}\cdot
    \sigma^2 d\left[\frac{1}{m-d-1} - \frac{C_1}{m+n}\right]
    =: \delta_{\mathrm{target}}. \label{eq:target-final}
\end{align}

\textbf{Dominance over Full-Corpus Fine-Tuning.}
From Theorem~\ref{thm:ev_for_ragft} and Proposition~\ref{prop:full-corpus}, the
dominance condition $\mathbb{E}[\mathcal{R}(\hat{\theta}_{\mathrm{FT}})] < \mathbb{E}[\mathcal{R}(\hat{\theta}_{\mathrm{full}})]$ requires, with $N_{\mathrm{full}} := Kn+m$,
\begin{align}
    &\sigma^2 + C_1\sigma^2\frac{d}{m+n}
    + C_2\delta\left(\frac{n}{m+n}\right)^{\!2}\left(1-\frac{d}{m+n}\right)^{\!-2}\Delta_{\max}^2
    \notag \\
    &< \sigma^2 + \frac{\sigma^2 d}{N_{\mathrm{full}}-d-1}
    \notag \\
    &\quad
    + \left(\frac{n}{N_{\mathrm{full}}}\right)^{\!2}\!(K-1)\bigl[1-(K-2)\rho\bigr]\Delta_{\min}^2.
    \label{eq:full-ineq-raw}
\end{align}
Cancelling $\sigma^2$ and applying \eqref{eq:spike-simplify} as before,
\eqref{eq:full-ineq-raw} becomes
\begin{align}
    C_1\sigma^2\frac{d}{m+n} + C_2\delta\,\frac{n^2}{(m+n-d)^2}\,\Delta_{\max}^2
    &< \frac{\sigma^2 d}{N_{\mathrm{full}}-d-1}
    \notag \\
    &\quad
    + \left(\frac{n}{N_{\mathrm{full}}}\right)^{\!2}\!(K-1)\bigl[1-(K-2)\rho\bigr]\Delta_{\min}^2.
    \label{eq:full-cancel-sigma}
\end{align}
Moving the variance term to the right-hand side of \eqref{eq:full-cancel-sigma} and
factoring out $\sigma^2 d$ from the matching terms,
\begin{align}
    C_2\delta\,\frac{n^2}{(m+n-d)^2}\,\Delta_{\max}^2
    &< \sigma^2 d\left[\frac{1}{N_{\mathrm{full}}-d-1} - \frac{C_1}{m+n}\right]
    \notag \\
    &\quad
    + \left(\frac{n}{N_{\mathrm{full}}}\right)^{\!2}\!(K-1)\bigl[1-(K-2)\rho\bigr]\Delta_{\min}^2.
    \label{eq:full-isolate}
\end{align}
Dividing both sides of \eqref{eq:full-isolate} by $C_2 n^2 \Delta_{\max}^2/(m+n-d)^2 > 0$,
\begin{align}
    \delta
    &< \frac{(m+n-d)^2}{C_2 n^2 \Delta_{\max}^2}
    \left\{\sigma^2 d\left[\frac{1}{N_{\mathrm{full}}-d-1} - \frac{C_1}{m+n}\right]
    \right.
    \notag \\
    &\qquad\qquad\qquad\quad
    \left.
    + \left(\frac{n}{N_{\mathrm{full}}}\right)^{\!2}\!(K-1)\bigl[1-(K-2)\rho\bigr]\Delta_{\min}^2\right\}
    =: \delta_{\mathrm{full}}. \label{eq:full-final}
\end{align}

\textbf{Intersection.}
RAG-FT strictly outperforms both baselines simultaneously if and only if
$\delta < \delta_{\mathrm{target}}$ from \eqref{eq:target-final} and
$\delta < \delta_{\mathrm{full}}$ from \eqref{eq:full-final} hold jointly, which is
equivalent to $\delta < \min\{\delta_{\mathrm{target}}, \delta_{\mathrm{full}}\}$. Since
$\delta$ itself satisfies the exponential bound of Theorem~\ref{thm:delta_ineq},
this condition is guaranteed whenever the task separation $\Delta_{\min}^2$ and sample
size $n$ are large enough that its right-hand side falls below
$\min\{\delta_{\mathrm{target}},\delta_{\mathrm{full}}\}$.
\end{proof}

\begin{remark}
\label{rmk:pos}
Both thresholds require $m > d+1$ (respectively $N_{\mathrm{full}} > d+1$) for the
exact inverse-Wishart denominators to be finite and positive, consistent with the
preconditions of Propositions~\ref{prop:target-only} and~\ref{prop:full-corpus}.
Whether each threshold is positive depends on the well-separated regime: $\delta_{\mathrm{target}}>0$
requires $\sigma^2 d/(m-d-1) > C\sigma^2 d/(m+n)$, a condition on the sample-size gap
between target-only and RAG-FT that holds whenever $m+n$ is not too close to the
interpolation threshold $m \approx d$; $\delta_{\mathrm{full}}>0$ additionally requires
the bias floor under Assumption~\ref{assumption:bounded-anticorrelation} to compensate
for the larger sample size of the full corpus, which holds once $\Delta_{\min}$
sufficiently exceeds the noise floor $\sigma\sqrt{d/N_{\mathrm{full}}}$, the same
well-separated-task regime required for $\delta$ itself to be small under
Theorem~\ref{thm:delta_ineq}.
\end{remark}

\section{Proofs for Section 5 (heteroscedastic regime (DPN))}
\label{sec:C}

\subsection{Proof of Theorem 4}
\label{subsec:C.1}

\OLSriskunderDPN*

\begin{proof}
Recall $x_{\mathrm{rag}}^{(i)} = x_q+r_i$, $r_i\sim\mathcal N(0,\delta_i^2I_d)$
independent across $i$, $\delta_i^2=\gamma i^q$, and write $s_n:=\sum_i
\delta_i^2$, $\mu:=m+s_n$, $\rho:=\|x_q\|_2^2$.

\bigskip
\textbf{Step 1: Decomposition of the pooled design.} Substituting the DPN
model into $X_{\mathrm{rag}}^\top X_{\mathrm{rag}}$ and expanding gives, exactly,
\begin{align}
X^\top X &= X_{\mathrm{target}}^\top X_{\mathrm{target}} + n\,x_qx_q^\top +
\sum_{i=1}^n r_ir_i^\top + x_qr_{\mathrm{sum}}^\top + r_{\mathrm{sum}}x_q^\top,
&r_{\mathrm{sum}}:=\sum_{i=1}^n r_i.
\end{align}
where $X_{\mathrm{target}} \in \mathbb{R}^{m \times d}$ denotes the target block of the pooled design.
Conditioning on $x_q$, the first block averages to $mI_d$, the third term
averages to $s_nI_d$, and the last two terms are mean zero, since
$\mathbb E[r_{\mathrm{sum}}\mid x_q]=0$. This gives
\[
\bar\Sigma:=\mathbb E[X^\top X\mid x_q]=\mu I_d+n\,x_qx_q^\top.
\]
Write $F:=X^\top X-\bar\Sigma$. This fluctuation splits into three mean-zero
pieces: the target block's Wishart fluctuation, of operator norm $O(1)$
with high probability since $m,d$ are fixed; the fluctuation of $\sum_i
r_ir_i^\top$ around $s_nI_d$, whose summands $S_i:=r_ir_i^\top-\delta_i^2I_d$
satisfy $\mathbb E[S_i^2]=\delta_i^4(d+1)I_d$ (using $\mathbb E[\|z\|^2zz^\top]
=(d+2)I_d$ for $z\sim\mathcal N(0,I_d)$), so the matrix variance proxy is
$v = \|\sum_i \mathbb E[S_i^2]\|_{\mathrm{op}} = (d+1)\sum_i\delta_i^4 =
(d+1)\gamma^2\sum_i i^{2q} = \Theta(d\,n^{2q+1})$, and the matrix Bernstein
inequality (Theorem~\ref{thm:matrix-bernstein}) gives operator norm
$O(\sqrt v)=O(\sqrt d\,n^{q+1/2})$ with high probability; and
the rank-two cross term, of operator norm $\Theta(\sqrt{s_n})=\Theta(n^{(q+1)/2})$.
Since $q+\tfrac12>\tfrac{q+1}2$ for every $q>0$, the middle piece dominates
(for $d$ fixed), giving
\begin{gather*}
\|F\|_{\mathrm{op}}=O(\sqrt d\,n^{q+1/2})\text{ w.h.p.}, \qquad
\lambda_{\min}(\bar\Sigma)=\mu=\Theta(n^{q+1}),
\\
\frac{\|F\|_{\mathrm{op}}}{\mu}=O(\sqrt d\,n^{-1/2})\to0.
\end{gather*}

\textbf{Step 2: Inverting the mean design.} Since $\bar\Sigma$ is scalar
plus rank one, the Sherman--Morrison identity (Theorem~\ref{thm:sherman-morrison})
gives, with $D:=\mu+n\rho$,
\[
\bar\Sigma^{-1}=\frac1\mu I_d-\frac{n}{\mu D}\,x_qx_q^\top.
\]
Because $\rho\sim\chi^2(d)=\Theta(1)$ with high probability and $\mu=
\Theta(n^{q+1})$ strictly dominates $n\rho=\Theta(n)$ for every $q>0$, we
have $D=\mu(1+O(n^{-q}))$, hence $1/\mu=\Theta(n^{-(q+1)})$ and $n/(\mu D)=
\Theta(n^{-(2q+1)})$.

\bigskip
\textbf{Step 3: The noise-weighted design.} The same argument applied to
$X^\top\Omega X$, with $\Omega$ diagonal with target entries $\sigma^2$ and
retrieved entries $\sigma_{\mathrm{rag},i}^2=\gamma\sigma^2i^q$, gives
\[
\bar\Sigma_\Omega:=\mathbb E[X^\top\Omega X\mid x_q]=\mu_\Omega I_d+s_\Omega\,x_qx_q^\top,
\]
with $\mu_\Omega=m\sigma^2+\sum_i\sigma_{\mathrm{rag},i}^2\delta_i^2=
\Theta(n^{2q+1})$ and $s_\Omega=\sum_i\sigma_{\mathrm{rag},i}^2=\Theta(n^{q+1})$.
The fluctuation $G:=X^\top\Omega X-\bar\Sigma_\Omega$ has summands carrying
an extra factor $\sigma_{\mathrm{rag},i}^2$ relative to Step 1, giving
variance proxy $\Theta(d\,n^{4q+1})$ and hence $\|G\|_{\mathrm{op}}=
O(\sqrt d\,n^{2q+1/2})$ w.h.p. by the identical matrix Bernstein argument,
so $\|G\|_{\mathrm{op}}/\mu_\Omega=O(\sqrt d\,n^{-1/2})\to0$.

\bigskip
\textbf{Step 4: The leading trace.} Since $\bar\Sigma$ and $\bar\Sigma_\Omega$
are both scalar-plus-rank-one in the direction $x_q$, they share an
eigenbasis: along $x_q/\|x_q\|$, the eigenvalues are $D$ and $\mu_\Omega+
s_\Omega\rho$ respectively; on the orthogonal complement, $\mu$ and
$\mu_\Omega$, each with multiplicity $d-1$. The product $\bar\Sigma^{-1}
\bar\Sigma_\Omega\bar\Sigma^{-1}$ is diagonal in this basis, so
\[
\bar V(x_q)=\frac{(d-1)\mu_\Omega}{\mu^2}+\frac{\mu_\Omega+s_\Omega\rho}{D^2}.
\]
Using $D=\mu(1+O(n^{-q}))$, the second term splits as
\[
\frac{\mu_\Omega+s_\Omega\rho}{D^2}=\frac{\mu_\Omega}{\mu^2}(1+O(n^{-q}))+
\frac{s_\Omega\rho}{\mu^2}(1+O(n^{-q})),
\]
and since $\mu_\Omega/\mu^2=\Theta(n^{-1})$ while $s_\Omega\rho/\mu^2=
\Theta(n^{-q-1})$ is strictly smaller for every $q>0$, this gives
$\bar V(x_q)=d\mu_\Omega/\mu^2\,(1+o(1))$. Substituting $\mu\sim\gamma n^{q+1}
/(q+1)$ and $\mu_\Omega\sim\gamma^2\sigma^2n^{2q+1}/(2q+1)$ gives the leading
term of the theorem.

\bigskip
\textbf{Step 5: Error control.} Writing $(X^\top X)^{-1}=\bar\Sigma^{-1}-
\bar\Sigma^{-1}F\bar\Sigma^{-1}+O(\|F\|_{\mathrm{op}}^2/\mu^3)$ and expanding
the trace $\mathrm{Tr}[(X^\top X)^{-1}(X^\top\Omega X)(X^\top X)^{-1}]$, the
purely linear terms in $F$ or $G$ vanish once we condition on $x_q$ and take
expectations, since both are mean zero. Because $F$ and $G$ are generated from the same retrieved offsets, the expansion contains correlated cross terms such as $FG$. However, all mixed $(F,G)$ terms are bounded by the same perturbative scale and are strictly lower order for fixed $d$. Specifically, we bound their contribution via the operator norm:
\[
\mathbb{E}|\operatorname{Tr}(\bar\Sigma^{-1}F\bar\Sigma^{-1}G\bar\Sigma^{-1})|
\le O\!\left(\frac{d\,\mathbb{E}[\|F\|_{\mathrm{op}}\|G\|_{\mathrm{op}}]}{\mu^3}\right)
= o(1/n).
\]
The leading surviving error is
\[
O\!\left(\frac{d\,\|F\|_{\mathrm{op}}^2\,\|\bar\Sigma_\Omega\|_{\mathrm{op}}}{\mu^4}\right)
=O\!\left(\frac{d\cdot d\,n^{2q+1}\cdot n^{2q+1}}{n^{4q+4}}\right)
=O(d^2n^{-2}),
\]
which is $o(1/n)$ for fixed $d$. Taking expectation over $x_q$ gives
$\mathbb E[V_{\mathrm{OLS}}(n)]=\frac{d\sigma^2(q+1)^2}{2q+1}\cdot\frac1n\,(1+o(1))$,
as claimed.
\end{proof}

\subsection{Proof of Theorem 5}
\label{subsec:C.2}

\LSAriskunderDPN*

\begin{proof}
The LSA forward pass is $\hat y_q=\frac1n\sum_{i=1}^n(x_q^\top x_{\mathrm{rag}}^{(i)})y_i$,
with $y_i=(x_{\mathrm{rag}}^{(i)})^\top\theta_k+\epsilon_i$. Write $v:=
\theta_k-\theta_{k^\star}$, $\tau:=\|v\|$, and $c:=s_n/n-1$.

\textbf{Step 1: Conditional bias.} Given $x_q$, $\mathbb E[x_{\mathrm{rag}}^{(i)}
(x_{\mathrm{rag}}^{(i)})^\top\mid x_q]=x_qx_q^\top+\delta_i^2I_d$, so
\[
\mathbb E[\hat y_q\mid x_q,\theta_k]=\frac1n\sum_i(\rho+\delta_i^2)\,x_q^\top\theta_k
=\Big(\rho+\frac{s_n}n\Big)x_q^\top\theta_k.
\]
Subtracting the target prediction $x_q^\top\theta_{k^\star}$ and writing
$a(\rho):=\rho+c$, this rearranges exactly to
\[
\mathrm{Bias}(x_q)=a(\rho)\,x_q^\top\theta_k+x_q^\top v.
\]

\textbf{Step 2: Second moment of the bias.} Write $x_q=\sqrt\rho\,\omega$
with $\omega$ uniform on the unit sphere, independent of $\rho:=\|x_q\|_2^2
\sim\chi^2(d)$. For fixed vectors $u,w$, $\mathbb E[(x_q^\top u)(x_q^\top w)
\mid\rho]=\rho(u\cdot w)/d$, so expanding $\mathrm{Bias}(x_q)^2$ and taking
expectation over $\rho$ using $\mathbb E[\rho]=d$, $\mathbb E[\rho^2]=d(d+2)$,
$\mathbb E[\rho^3]=d(d+2)(d+4)$ gives, exactly,
\begin{align}
\mathbb E[\mathrm{Bias}(x_q)^2]&=\|\theta_k\|^2\big[(d+2)(d+4)+2c(d+2)+c^2\big]\notag\\
&\quad+2(\theta_k\cdot v)\big[(d+2)+c\big]+\tau^2.
\end{align}
Since $c=s_n/n-1\sim\gamma n^q/(q+1)\to\infty$ for every fixed $q>0$, the
$c^2$ term dominates the $c^1$ and $c^0$ terms, giving
\[
\mathbb E[\mathrm{Bias}(x_q)^2]=\|\theta_k\|^2c^2\,(1+o(1))=
\frac{\gamma^2\|\theta_k\|^2}{(q+1)^2}\,n^{2q}\,(1+o(1))=\Theta(n^{2q}).
\]
At $\theta_k=\theta_{k^\star}$ this reduces to $\|\theta_{k^\star}\|^2c^2\,
(1+o(1))$, which still diverges: the bias term does not depend on retrieval
outcome.

\textbf{Step 3: Leading-order conditional variance.} Expanding
$(x_q^\top x_{\mathrm{rag}}^{(i)})y_i$ with $A_i := x_q^\top r_i$ and
$B_i := r_i^\top\theta_k$, the fluctuation around its conditional mean
splits into $A_iB_i + A_i\epsilon_i$ -- quadratic in the retrieved offset
$r_i$, of variance $\Theta(\delta_i^4)$ -- plus three terms linear in $r_i$
or $\epsilon_i$ ($\rho B_i$, $\rho\epsilon_i$, $A_ix_q^\top\theta_k$), of
variance $\Theta(\delta_i^2)$ each. Since $\delta_i^2=\gamma i^q$, after the
$1/n^2$ normalization the linear pieces sum to $\Theta(n^{q-1}) =
o(n^{2q-1})$ against the $\Theta(n^{2q-1})$ rate established below, so they
are dropped here and absorbed into the theorem's $(1+o(1))$. Since $r_i$
and $\epsilon_i$ are independent across $i$ and of $x_q$, the retained terms
$T_i := A_iB_i + A_i\epsilon_i$ are independent across $i$ given $x_q$, so
\[
\mathrm{Var}(\hat y_q\mid x_q) = \frac1{n^2}\sum_i \mathrm{Var}(T_i\mid x_q)
= \frac1{n^2}\sum_i\Big[\mathrm{Var}(A_iB_i\mid x_q) +
\mathbb{E}[A_i^2\mid x_q]\,\sigma_{\mathrm{rag},i}^2\Big],
\]
using $\mathrm{Cov}(A_iB_i, A_i\epsilon_i\mid x_q) = 0$ (odd moment of
$\epsilon_i$) and $\mathbb{E}[A_i\epsilon_i\mid x_q] = 0$.

Since $r_i\sim\mathcal N(0,\delta_i^2 I_d)$ given $x_q$, $A_i$ and $B_i$ are
jointly Gaussian with $\mathbb{E}[A_i^2]=\delta_i^2\rho$,
$\mathbb{E}[B_i^2]=\delta_i^2\|\theta_k\|^2$,
$\mathbb{E}[A_iB_i]=\delta_i^2(x_q^\top\theta_k)$. The Gaussian identity
$\mathrm{Var}(AB) = \mathbb{E}[A^2]\mathbb{E}[B^2] + \mathbb{E}[AB]^2$ for
jointly mean-zero Gaussian $A,B$ gives
\[
\mathrm{Var}(A_iB_i\mid x_q) = \delta_i^4\rho\|\theta_k\|^2 +
\delta_i^4(x_q^\top\theta_k)^2.
\]
Taking expectation over $\rho\sim\chi^2(d)$ as in Step~2, using
$\mathbb{E}[\rho]=d$ and $\mathbb{E}[\rho(x_q^\top\theta_k)^2] =
(d+2)\|\theta_k\|^2$ (by the same $\omega$-uniform-on-sphere argument),
\[
\mathbb{E}[\mathrm{Var}(A_iB_i\mid x_q)] = (d+1)\,\delta_i^4\,\|\theta_k\|^2.
\]
Combining with $\mathbb{E}[A_i^2\mid x_q]\sigma_{\mathrm{rag},i}^2 =
\delta_i^2\rho\,\sigma_{\mathrm{rag},i}^2$, whose expectation over $\rho$ is
$d\,\delta_i^2\sigma_{\mathrm{rag},i}^2$, the leading-order finite-$n$ variance is
\begin{equation}
\label{eq:lsa-exact-variance}
\mathbb{E}[\mathrm{Var}(\hat y_q\mid x_q)] = \frac{(d+1)\|\theta_k\|^2}{n^2}
\sum_{i=1}^n \delta_i^4 + \frac{d}{n^2}\sum_{i=1}^n
\delta_i^2\sigma_{\mathrm{rag},i}^2.
\end{equation}

\textbf{Step 4: Asymptotic rate.} Substituting $\delta_i^2=\gamma i^q$,
$\sigma_{\mathrm{rag},i}^2=\gamma\sigma^2 i^q$ into
\eqref{eq:lsa-exact-variance} and using $\sum_i i^{2q} =
\Theta(n^{2q+1})$,
\[
\mathbb{E}[\mathrm{Var}(\hat y_q\mid x_q)] =
\frac{\gamma^2\big[(d+1)\|\theta_k\|^2 + d\sigma^2\big]}{2q+1}\,n^{2q-1}\,
(1+o(1)) = \Theta(d\cdot n^{2q-1}),
\]
where the $(d+1)\|\theta_k\|^2$ (retrieved-offset) and $d\sigma^2$
(label-noise) contributions are the same order in $n$, so neither can be
dropped.

\textbf{Step 5: Combining terms.} Comparing the two growing terms using the
exact bias formula from Step~2 and the leading-order variance formula
\eqref{eq:lsa-exact-variance},
\[
\frac{\mathbb{E}[\mathrm{Bias}(x_q)^2]}{\mathbb{E}[\mathrm{Var}(\hat
y_q\mid x_q)]} = \Theta\!\left(\frac{n^{2q}}{d\cdot n^{2q-1}}\right) =
\Theta\!\left(\frac{n}{d}\right)\to\infty
\]
for fixed $d$, so the bias term dominates for every $q>0$. Summing
$\sigma^2 + \mathbb{E}[\mathrm{Bias}(x_q)^2] +
\mathbb{E}[\mathrm{Var}(\hat y_q\mid x_q)]$ gives $\mathcal{R}_{\mathrm{LSA}}(n) =
\sigma^2 + \Theta(n^{2q}) + \Theta(d\cdot n^{2q-1})$, which diverges as
$n\to\infty$ for every fixed $q>0$.

\end{proof}
\section{Additional empirical validation}
\label{app:empirical-details}

This section collects settings and estimator definitions omitted from
Section~\ref{sec:empirics} for space.

The divergence of the literal LSA forward pass under distance-proportional noise carries no leading-order dependence on ambient dimension, and is specific to unweighted aggregation rather than to retrieval noise itself. Figure~\ref{fig:app_no-ceiling-reweight} confirms both properties directly.

\begin{figure}[htbp]
    \centering
    \begin{subfigure}[b]{0.48\textwidth}
        \includegraphics[width=\textwidth]{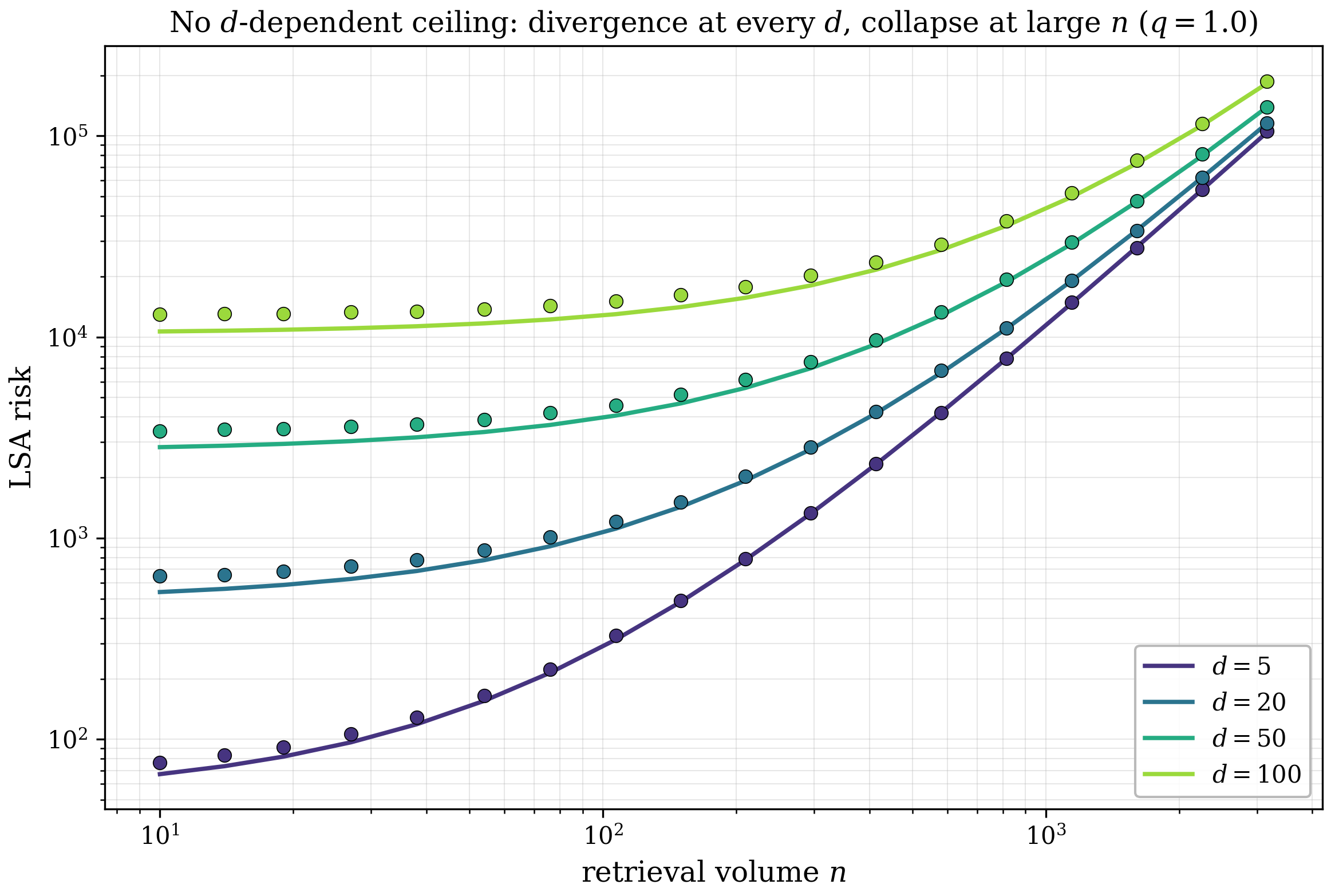}
        \caption{No dimension-dependent ceiling.}
        \label{fig:app_no-ceiling}
    \end{subfigure}
    \hfill
    \begin{subfigure}[b]{0.48\textwidth}
        \includegraphics[width=\textwidth]{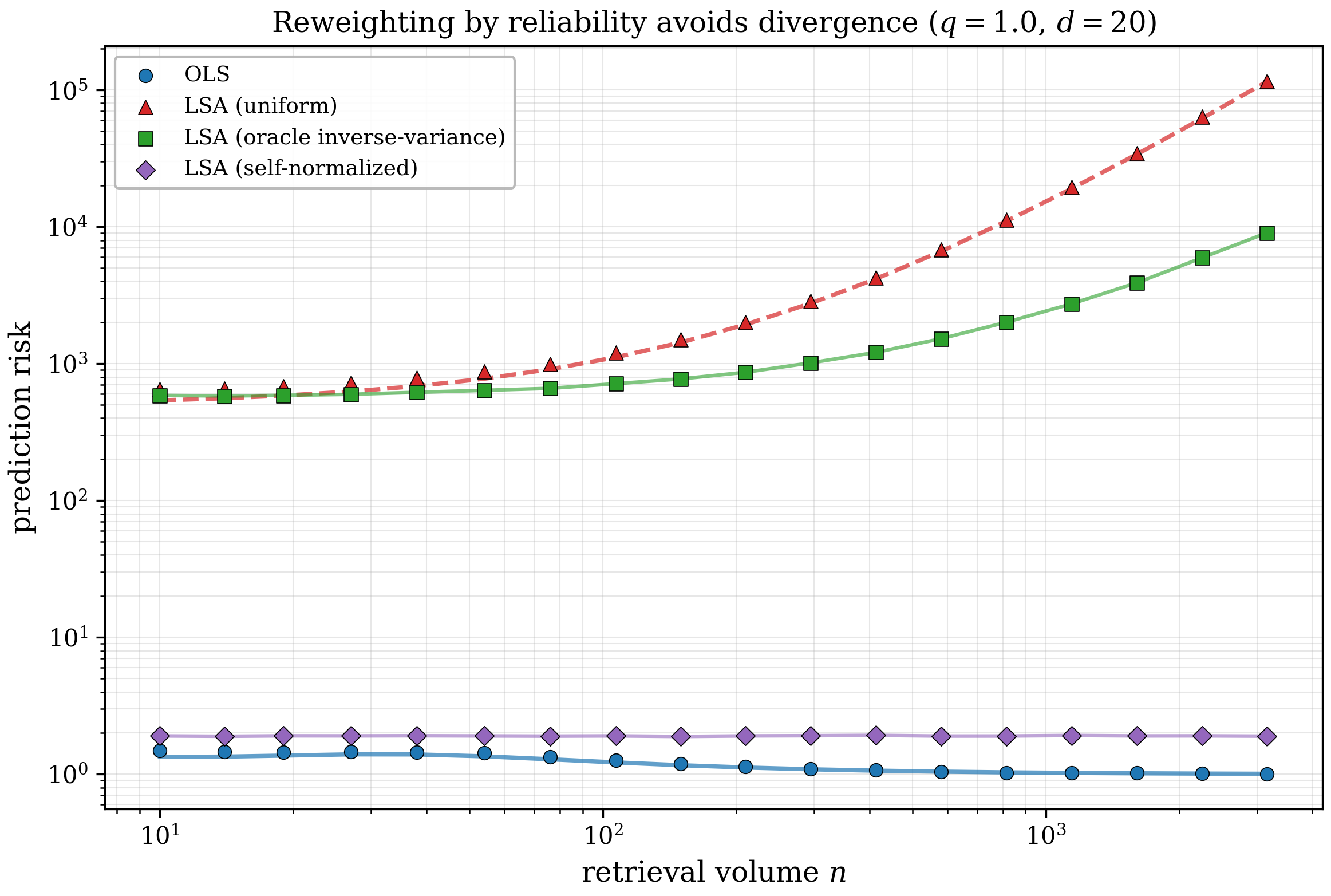}
        \caption{Reweighting recovers stability.}
        \label{fig:app_reweight}
    \end{subfigure}
    \caption{(a) LSA risk vs.\ $n$ across $d \in \{5,20,50,100\}$: curves
    separate where variance scales with $d$, then collapse as the
    dimension-independent bias term takes over. (b) Oracle inverse-variance
    weighting and self-normalization both remain stable where uniform
    averaging diverges.}
    \label{fig:app_no-ceiling-reweight}
\end{figure}

\paragraph{Methodology.} All Monte Carlo points are obtained by directly
simulating the generative model of Section~\ref{subsec:dpn-model} and fitting or evaluating
each estimator from its literal defining formula (least-squares for OLS,
the explicit forward-pass average for LSA and its variants); no Monte Carlo
point is computed by numerically evaluating the closed-form risk
expressions plotted as theory curves. Each point in Figure~\ref{fig:headline} and~\ref{fig:app_no-ceiling}
averages held-out prediction risk over at least 300 independent seeds, with
fresh test queries drawn per seed; Figure~\ref{fig:app_reweight} averages over 300 seeds per
retrieval volume $n$. Error bars are standard errors across seeds, clipped
below at $0.1\times$ the mean so that the lower whisker stays positive on the
log axis.

\paragraph{Interpretation.} The plotted curves are the
\emph{finite-$n$} risk expressions derived in
Appendix~\ref{subsec:C.1}--\ref{subsec:C.2}, namely the semi-analytic
$\mathbb{E}[V_{\mathrm{OLS}}(n)]$ of Theorem~\ref{thm:ols-dpn} and the exact
bias plus leading-order variance of Theorem~\ref{thm:lsa-dpn}, not the
asymptotic rates $\Theta(d/n)$ and $\Theta(n^{2q})$ those expressions reduce
to. This matters for interpretation: the theorems are statements about
$n\to\infty$, so plotting the rates alone would leave the small-$n$ region
untested, whereas the finite-$n$ expressions are falsifiable at every $n$ on
the grid. The LSA curve is therefore non-monotone at small $n$, where the
$c^0$ and $c^1$ terms of the exact bias still compete with $c^2$; the
$\Theta(n^{2q})$ slope emerges only once $c=s_n/n-1$ grows large.

\paragraph{Reweighted LSA estimators.} Figure~\ref{fig:app_reweight} evaluates two variants of
the forward pass that replace uniform averaging with a reliability-weighted
average,
\[
\hat y_q = \frac{\sum_{i=1}^n w_i\,(x_q^\top x_{\mathrm{rag}}^{(i)})\,y_i}
{\sum_{i=1}^n w_i \cdot z_i},
\]
differing only in the weights $w_i$ and normalizer $z_i$:
\begin{itemize}
\item \emph{Oracle inverse-variance}: $w_i = 1/\sigma_{\mathrm{rag},i}^2$
with $z_i=1$, i.e., the estimator is given the true DPN noise variance at
each rank and normalizes by $\sum_i w_i$.
\item \emph{Self-normalized}: $w_i=1$ with $z_i = (x_q^\top
x_{\mathrm{rag}}^{(i)})^2$, i.e., the estimator normalizes by the total
similarity mass $\sum_i(x_q^\top x_{\mathrm{rag}}^{(i)})^2$ rather than by
$n$, using no information beyond what the uniform forward pass already
computes.
\end{itemize}

\paragraph{Parameter settings.} The table provided below lists the
fixed parameters underlying each figure; $n$ is swept over a
log-spaced grid in each case, and $\theta_{k^*}$ is a fixed unit vector
throughout.

\begin{table}[h]
\centering
\begin{tabular}{lccccc}
\toprule
Figure & $d$ & $m$ & $q$ & $\gamma$ & $n$ range \\
\midrule
1a     & 20            & 50 & 0.5 & 0.3 & $10$--$10^4$ \\
1a     & 20            & 50 & 1   & 0.2 & $10$--$10^{3.5}$ \\
1a     & 20            & 50 & 2   & 0.1 & $10$--$10^{2.5}$ \\
1b     & 20            & 50 & 1             & 0.2               & $10$--$10^{3.5}$ \\
2a       & $\{5,20,50,100\}$ & 50 & 1         & 0.2               & $10$--$10^{3.5}$ \\
2b       & 20            & 50 & 1             & 0.2               & $10$--$10^{3.5}$ \\
\bottomrule
\end{tabular}
\caption{Fixed parameters per figure. $\sigma=1$ throughout. Panel 1a sweeps
three $(q,\gamma)$ pairs, each over its own $n$ range (narrower at larger
$q$, since risk grows faster and the largest $n$ values become numerically
extreme).}
\label{tab:empirics-params}
\end{table}


\newpage

\end{document}